\documentclass[11pt]{article}

\usepackage{amsmath,amssymb,amsthm}

\usepackage{times}
\usepackage{bm}
\usepackage{natbib}
\usepackage{algorithm}
\usepackage{amssymb}
\usepackage{mathrsfs}
\usepackage{graphicx}
\usepackage[flushleft]{threeparttable}
\usepackage{enumitem} 
\usepackage{color}
\usepackage{mathtools}
\usepackage{subfigure}

\theoremstyle{definition}
\newtheorem{thm}{Theorem}[section]
\newtheorem{defi}[thm]{Definition}
\newtheorem{remark}[thm]{Remark}
\newtheorem{prp}[thm]{Proposition}
\newtheorem{crl}[thm]{Corollary}

\usepackage{star}

\usepackage[colorlinks,
linkcolor=blue,
anchorcolor=blue,
citecolor=blue
]{hyperref}

\newcommand{\E}{\mathbb E}

\def\##1\#{\begin{align}#1\end{align}}
\def\$#1\${\begin{align*}#1\end{align*}}

\newcommand{\Rom}[1]{\text{\uppercase\expandafter{\romannumeral #1\relax}}}

\usepackage{txfonts}

\usepackage{geometry}

\begin{document}

\title{ \LARGE Online Linearized LASSO}     

\author{Shuoguang Yang,\thanks{Department of Industrial Engineering and Decision Analytics, The Hong Kong University of Science and Technology, Clear Water Bay, Hong Kong SAR; E-mail: \texttt{yangsg@ust.hk}.}  \and 
Yuhao Yan\thanks{Department of Industrial Engineering and Decision Analytics, The Hong Kong University of Science and Technology, Clear Water Bay, Hong Kong SAR; E-mail: \texttt{yyanaz@connect.ust.hk}.} , 
\and 
Xiuneng Zhu\thanks{Tower Research Capital LLC; E-mail: \texttt{xiuneng.zhu@gmail.com}},
\and 
Qiang Sun\thanks{Department of Statistical Sciences, University of Toronto, ON, Canada; E-mail: \texttt{qiang.sun@utoronto.ca}.}}
\date{ }

\maketitle

\vspace{-0.25in}

\begin{abstract}
Sparse regression has been a popular approach to perform variable selection and enhance the prediction accuracy and interpretability of the resulting statistical model. Existing approaches focus on offline regularized regression, while the online scenario has rarely been studied. In this paper, we propose a novel online sparse linear regression framework for analyzing streaming data when data points arrive sequentially. Our proposed method is memory efficient and requires less stringent restricted strong convexity assumptions. Theoretically,  we show that with a properly chosen regularization parameter, the $\ell_2$-norm statistical error of our estimator diminishes to zero in the optimal order of $\widetilde \cO({\sqrt{s/t}})$,
where $s$ is the sparsity level, $t$ is the streaming sample size, and $\widetilde \cO(\cdot)$ hides  logarithmic terms. Numerical experiments demonstrate the practical efficiency of our algorithm.
\end{abstract}


\section{Introduction} 
\label{sec:intro}
With the development of modern data acquisition technologies,  high-dimensional statistics has attracted significant interest  due to its ability in handling datasets with large numbers of features. 
To alleviate the challenges introduced by massive amounts of features, a popular assumption is the sparsity assumption, that is,  only a few variables contribute to the response.   
A common approach that exploits such sparsity assumption is to solve  the following penalized empirical risk minimization problem 
\begin{equation}
   \hat \beta =  \argmin_{\beta \in \R^p} \Big \{ \cL(\beta) + \cR(\beta, \lambda)\Big \}, 
\end{equation}
where $\cL(\beta)$ is the empirical loss function, $\beta$ is a $p$-dimensional coefficient vector where $p$ is the covariate dimension, and $\cR(\beta,\lambda)$ is a sparsity-inducing regularizer with $\lambda$ being the regularization parameter.  
The regularizer $\cR(\beta;\lambda)$ can  be  either convex, such as the LASSO penalty $\cR(\beta;\lambda) = \lambda \| \beta\|_1 $, or nonconvex such as the SCAD penalty \citep{tibs1996regression, fan2001variable} {or the MCP~\citep{zhang2010nearly}.}

Studies in high-dimensional statistics can be divided into two major streams, the offline and online sparse problems, based on the data sampling mechanisms. Originated from \cite{tibs1996regression}, the offline problem considers an environment where all data are available at the beginning and the decision-maker aims to compute a good estimator by solving the corresponding regularized optimization problem. This problem have been previously studied  from both statistical and computational perspectives \citep{agarwal2012fast,loh2013regularized,fan2018lamm}. 

In contrast, online sparse regression considers the environment  where data are not readily accessible at the beginning but arrive sequentially. Instead of solving for one final estimator in the offline setting, the online  problem  requires computing a sequence of estimators $\{ \hat \beta_t\}$ such that the working  estimator can be computed efficiently whenever additional samples arrive. This requires a significantly amount of effort and has been less studied in high dimensional statistics \citep{kale2014open, kale2017adaptive, fan2018statistical}. 

In this paper, we consider the online environment where a feature-label pair $(x_t,y_t)$ comes in each round $t \geq 1$. Here $x_t \in \R^p$ is the covariate and $y_t$ is the response  depending  on both $x_t$ and an unknown coefficient $\beta^* \in \R^p$. Letting $S = \text{support}(\beta^*)$ be the support of $\beta^*$, we assume $\beta^*$ preserves a sparse structure whose sparsity level (cardinality of support) $s$ is much smaller than the dimension $p$, i.e., $ s = |S| = \| \beta^*\|_0\ll p$. Our target is to compute an estimator $\beta_t$ in each round $t$ to estimate the underlying sparse coefficient $\beta^*$.  

Although an extensive amount of effort has been made to study online optimization in the low-dimensional setting \citep{kushner:1997,rakhlin2011making}, online high-dimensional sparse optimization has been much worse understood, due to the extra sparse-inducing regularizer $\cR(\beta ,\lambda)$. It mainly suffers from three key challenges. (i) Memory and storage challenge: Naively utilizing the entire dataset up to time $t$ incurs  a cost of $\cO(pt)$ storage and memory complexity when computing the solution sequence. 
(ii) Dynamic update of the regularization parameter: To ensure the best statistical performance of the estimator sequence at each time $t$,  the data-dependent regularization parameter $\lambda$  needs to be updated 
in each online round as the sample size $t$ increases when new data arrive. Most sparse online optimization algorithms \citep{langford2009sparse, xiao2009dual} for regularized problems do not apply to our settings since they only consider the fixed-$\lambda$ optimization problems. Part of the reason is that they did not consider optimal statistical guarantees such as the parameter estimation error or other optimality measures.  
(iii) Restrictive strong convexity: In the online setting, to ensure each of the estimator sequence to be sparse, one needs the restrictive strong convexity (RSC) condition to hold uniformly for all online learning rounds.  This can be rather stringent in practice because the loss function varies in each round. Therefore, it  remains a challenge whether a memory efficient framework could be developed for online sparse optimization with optimal statistical guarantees~\citep{kale2014open}. 

In this paper, we propose an online algorithm which in round $t$ solves the following $\ell_1$-regularized problem 
\begin{equation*}
\hat{\beta}_t = \argmin_{\beta \in \mathbb{R}^p} \Big \{ L_t(\beta; \hat \beta_{t-1}) + \lambda_t \Vert \beta\Vert_1 \Big \}.
\end{equation*}
where 
\begin{equation*}
\begin{split}
   L_t (\beta; \hat \beta_{t-1})  = \underbrace{  \ell_0(\beta)  }_{\text{squared loss}}   + \underbrace{ \Big \langle \sum^t_{j = 1}  \frac{w_{t,j}  }{W_t} \nabla \ell_j(\hat \beta_{t-1})  - \nabla \ell_0( \hat \beta_{t-1} ), \beta \Big \rangle }_{\text{linearized loss}},  
  \end{split}
\end{equation*}
is  the  loss function, consisting of a squared loss $l_0(\beta)$ of an initial batch of sample size $t_0$, and a linearized loss  evaluated at current best estimator $\hat \beta_{t-1}$ 
with data received before up to rounds $t$. 
Notably, our framework enjoys a low memory cost of $\cO(p\, t_0 + p^2)$ that does not depend on the online round $t$, and only requires the RSC condition to hold for the initial loss $\ell_0(\beta)$ instead of all online rounds.  We show that under mild conditions,  the entire solution trajectory $\{ \hat \beta_t\}$ falls within a restricted cone, and is consistent such that it converges  to the underlying coefficient vector at the best possible convergence rate. To our best knowledge, this is the first algorithm that enjoys the above properties. 

\textbf{Contributions.} We make the following three major contributions.
\begin{enumerate}
    \item[(i)] We propose a memory efficient scheme for online sparse linear regression that iteratively solves an $l_1$-regularized optimization problem. Our framework adopts a novel loss function that consists of a squared loss of an initial batch of sample size $t_0$ and a linearized loss, which enjoys a fixed $\cO(t_0 p + p^2)$ memory cost. Meanwhile, our resulting loss function satisfies the RSC condition throughout all online rounds as long as  it holds for the initial batch. This is a siginifit boost as we do not need it to hold uniformly for all online learning rounds. 
    \item[(ii)] We show that by properly choosing the regularization parameters, our estimators are consistent whose $\ell_1$-norm parameter estimation errors decay to zero in the optimal rate of $\tilde \cO({\sqrt{s/t}})$,  where $t$ is the streaming sample size and $s$ is the sparsity level. 
    \item[(iii)] We conduct numerical experiments to test the practical performance of our algorithm against other baseline algorithms under various settings. Numerical results demonstrate the practical efficiency of our algorithm and validates our theoretical results. 
\end{enumerate}

\subsection{Related Work}
There has been relatively less work for online sparse regressions, which has different flavors from ours. We discuss those that are mostly related to our work. 

\textbf{Optimization perspective:} As mentioned previously, previously sparse online optimization algorithms \citep{xiao2009dual,langford2009sparse,bertsekas2011incremental,duchi2011adaptive} considers fixed-$\lambda$ optimization problems and thus do not apply to our settings. Part of the reason is that they focused on regret bound for fixed $\lambda$ and
did not study the optimal statistical performances of the estimator sequence in a statistical setting. For the regularized optimization problem \eqref{eq:offline}, the data-dependent regularized parameter has to be updated in each online round, and this makes the statistical analysis challenging. Another approach that avoids this challenge is the online sparsity constrained optimization, which however brings additional computational intractabilities as it is commonly believed sparsity contrained optimization is NP hard \citep{foster2015variable}. \cite{kale2014open} raised the open question that whether it is possible to design an efficient algorithm for the online sparse regression problem to achieve a sublinear regret bound. Toward addressing this challenge, \cite{kale2017adaptive} proposed to solve a sequence of Dantzig selector problems in an online manner, which achieved a sublinear regret bound. However, their result requires a bound on the Restricted Isometry Property constant $\leq 1/5$ uniformly over all online rounds, or equivalently a bound on the condition number $\leq 3/2$ uniformly. This is undesirable because real-world high-dimensional data analyses often require estimation methods under arbitrarily large condition numbers. 
Moreover, they focused on the regret bound instead of the parameter estimation error or the prediction error, and the latter is typically used in the statistics literature.

\textbf{Statistical perspective:}
Closely related to our work, \cite{fan2018statistical} proposed a two-stage algorithm that first conducts a burn-in stage that identifies the support of the sparse underlying coefficient through solving an offline LASSO problem, and conducts the online learning stage that employs a fixed number of truncated gradient descent steps onto the pre-identified support set upon receiving each online data. Unfortunately,  they require to  identify the true support without errors  within the first burn-in stage, which further relies on a minimal signal strength assumption and a sufficient large initial data batch so that the true support can be correctly determined with high probability. Meanwhile, this approach also requires the loss functions to satisfy the restricted strong convexity condition uniformly over all online learning rounds. These assumptions are rather stringent in practice. 

\section{From Offline to Online}

\textbf{Notadation.} We first summarize the notation used throughout the paper. For any positive integer $K$, we write $[K] = \{ 1,2,\cdots, K\}$, the collection of positive integers up to $K$. 
For any two sequences of positive real numbers $\{ a_n\}$ and $\{ b_n\}$, we write 
$a_n \lesssim b_n$ if there exist a constant $C>0$ and a positive integer $n_0$ such that $a_n \leq C b_n$ for all $n \geq n_0$. 

This section develops an algorithm for online sparse linear regression. We consider an online environment where data arrive sequentially, and we only have one single machine that has limited storage and memory. Assume the covariate vector $x_i \in \R^p$  and the response $y_i \in \R$ follows the linear regression model: 
$$ y_{j}= x_{j}^\top \beta^*+\epsilon_{j},$$ 
where $\beta^* \in \R^p$ is the underlying sparse regression coefficient vector such that $\| \beta^*\|_0 \ll p$ and  $\epsilon_{j}$ is a random noise. 

To estimate the underlying coefficient $\beta^*$ under the linear model, a commonly adopted loss function is the squared loss $\ell_j$, which calculates the squared error for each data point $(x_i,y_i)$ that 
\begin{equation*}\label{eq:loss_each_batch}
\ell_j(\beta) = \frac{1}{2}  (y_{j}-  x_{j}^\top \beta)^2.
\end{equation*}
We consider an online environment where we have access to an initial batch consisting of $t_0$ independently generated data points  $(x_{0j}, y_{0j})_{1 \leq j \leq t_0}$. For the initial batch, with a slight abuse of notation, we define its corresponding loss function as 
\begin{equation*}
\ell_0(\beta) = \frac{1}{2t_0} \sum^{t_0}_{j = 1} (y_{0j} -  x_{0j}^\top \beta)^2,
\end{equation*}
which is the averaged loss over each data point within the initial batch. Subsequently, one data point $(x_t,y_t)$ comes in each online learning round $t$. 

In high-dimensional statistics, the LASSO approach has been widely employed to obtain sparse estimators. 
Ideally, if our machine has infinite memory and storage, after the reception of the first $t$ data points as well as the initial batch, the offline Lasso estimates the regression coefficient vector $\beta^*$ by directly solving
\begin{equation} \label{eq:offline}
  \min_{\beta \in \R^p} \Big \{ \frac{1}{t+t_0} \Big ( t_0 \ell_0(\beta) +\sum_{j=1}^t\ell_j(\beta) \Big ) + \lambda_{t} \Vert \beta\Vert_1 \Big \}.  
\end{equation}

The above LASSO formulation works well for the offline problem, but for the online scenario, it suffers from three key challenges mentioned in Section~\ref{sec:intro}, namely memory and storage restriction, update of regularization parameter, and uniform restrictive strong convexity condition. 
Because of the these issues,  online sparse linear regression has become particularly challenging.

To overcome the above issues, we propose a novel memory efficient framework for solving the online sparse regression problem.  
Our new framework  replaces all the loss functions for the online data, excluding the initial batch, by their linear approximation, approximated using the current best available estimator, denoted by $\tilde \beta$. 
We assign a fixed weight to the initial batch and allows general weights for the data obtained for each online learning round. Specifically, 
let $w_{t,j}$ be the weight for the $j$-th data $(x_j,y_j)$ and let $W_t = \sum_{j=1}^t w_{t,j}$ be the total weights, we define the loss function as 
\begin{equation}\label{def:loss}
 L_t (\beta;\tilde  \beta) =   \ell_0(\beta)   - \nabla \ell_0(\tilde  \beta)^\top  \beta  + \Big \langle \sum^t_{j = 1}  \frac{w_{t,j}  }{W_t} \nabla \ell_j(\tilde  \beta) , \beta \Big \rangle,  
\end{equation}
where we have singled out the initial batch and distribute the rest as linear approximations. Here we call $\tilde \beta$ the \emph{root} of the loss function to evaluate the linear approximation terms. 
In our formulation, the quadratic term $l_0(\beta)$ resulted from the initial batch helps provide the curvature of the loss function, while the linear approximation terms utilizes data from the subsequent online learning rounds, which would not affect the curvature but help us improve the statistical accuracy of the obtained solution at the current online round. Note that for the initial batch where $t=0$, we do not introduce any linear approximation and the above loss function reduces to the standard least-squared loss. 

To better understand the above loss function, let us calculate its gradient at $\beta$:
\begin{equation*}
\begin{split}
  & \nabla L_t (\beta;\tilde  \beta)   =  \nabla \ell_0(\beta)   - \nabla \ell_0(\tilde  \beta)  + \sum^t_{j = 1}  \frac{w_{t,j}  }{W_t} \nabla \ell_j(\tilde  \beta) . 
  \end{split}
\end{equation*}
Although the initial batch is fixed, the term $\nabla \ell_0(\beta)   - \nabla \ell_0(\tilde  \beta)$ is small  provided  $\tilde \beta $ is close to $\beta$. Intuitively, if we can improve our root estimator and find a sequence $\{ \tilde  \beta_t \}$ converging to the underlying coefficient $\beta^*$, then 
\begin{equation*}
\begin{split}
\nabla L_t (\beta;\tilde  \beta_t)   \approx \sum^t_{j = 1}  \frac{w_{t,j}  }{W_t} \nabla \ell_j(\tilde   \beta_t) \approx \sum^t_{j = 1}  \frac{w_{t,j}  }{W_t} \nabla \ell_j(\beta^*),
  \end{split}
\end{equation*}
whose decay rate depends on the online data $ \{ (x_i,y_i) \}_{i=1}^t$ but \emph{not} the initial batch. 

As we have seen, the linear term $ - \nabla \ell_0(\tilde  \beta)^\top \beta$ introduced in the loss function \eqref{def:loss} helps us break the bottleneck induced by the initial batch $\ell_0(\beta)$ without rescaling it as $\frac{t_0}{t_0 + t} \ell_0(\beta)$, which is commonly adopted by LASSO formulation \eqref{eq:offline} but would lose the curvature as the rescaling factor $\frac{t_0}{t_0 + t} \to 0 $ when $t\to \infty$. This makes our framework memory efficient and preserves the curvature information simultaneously. 

In addition, because the curvature of our loss function is provided by the initial batch instead of online data, as we will discuss in Section~\ref{sec:theory}, the loss functions in all online rounds   satisfy the RSC condition uniformly once it holds for the initial batch squared-loss $\ell_0(\beta)$. This is another  advantage of our new framework. 

\subsection{Algorithm}
After introducing the loss function, we formally state our algorithm for solving streaming sparse regression. 
Our online Lasso algorithm consists of the following two stages. In the first stage, we calculate our initial estimator as
\begin{equation}\label{def:beta_0}
\hat{\beta}_0 = \argmin_{\beta \in \mathbb{R}^p} \Big \{ \ell_0(\beta) + \lambda_0 \Vert \beta\Vert_1 \Big \}.
\end{equation}
Our second stage recursively solves a $l_1$-regularized optimization problem,  with $\hat \beta_0$ being the initialization point.  Specifically, letting $\hat \beta_{t-1}$ be the optimal solution obtained for iteration $t-1$, upon receiving the $t$-th data point, we adopt $L_t (\beta;\hat \beta_{t-1}) $ as the loss function and solve 
\begin{equation}\label{def:beta_t}
\begin{split}
\hat{\beta}_{t} =\argmin_{\beta \in \R^p}  \Big \{ L_t (\beta;\hat \beta_{t-1})   + \lambda_{t} \Vert \beta\Vert_1  
\Big \}.
\end{split}
\end{equation}
We collect the algorithm in  Algorithm~\ref{alg:1} and refer  to this approach as the Online Linearized LASSO (OLin-LASSO). 

Note the the loss function preserves a simple structure that it is the sum of a quadratic function and a linear term, which can be efficiently solved by various algorithms, such as FISTA~\citep{beck2009fast}. 

\begin{algorithm}[t]
\caption{Online Linearized (OLin) LASSO}\label{alg:1}
\begin{algorithmic}[1]
\REQUIRE $\{\lambda_t\}$, initial batch size $t_0$, online rounds $T$
\STATE generate $t_0$ samples, 
\STATE compute  $\hat \beta_0$ by \eqref{def:beta_0}. 
\FOR{$j= 1, \cdots, t$}
\STATE  receive the data pair $(x_j,y_j)$.
\STATE compute $\hat \beta_{j}$ by \eqref{def:beta_t}. 
 \ENDFOR
 \RETURN $\{\hat \beta_j \}_{j=1}^t$
\end{algorithmic}
\end{algorithm}


\subsection{Memory Efficient Weighting Scheme}
Before proceeding, let us briefly discuss the memory cost and updating schemes under our framework. 
For the squared loss, our proposed algorithm only requires minimal memory and storage space, independent of $t$. We discuss this by considering two cases in the following. \\

\noindent{\bf $t$-independent weights}: {For  $t$-independent  weights $w_{t,j}=w_j$, the proposed  algorithm only needs $\cO(t_0 p+p^2)$ memory  space to record summary statistics from the history, which is independent of $t$. 
Note that 
\begin{equation*}
    \begin{split}
    \sum^t_{j = 1}  \frac{w_{t,j}  }{W_t} \nabla \ell_j(  \beta)  = \sum^t_{j = 1}  \frac{w_{t,j}  }{W_t} ( x_j x_j^\top \beta - x_j y_j ).
    \end{split}
\end{equation*}
Therefore, we could store the terms $s_t =\sum^t_{j = 1}  w_{t,j}   x_j x_j^\top  $ and $r_t = \sum^t_{j = 1}  w_{t,j}  x_j y_j$ to evaluate the gradient. Specifically, upon receiving the data point $(x_t,y_t)$, to evaluate the cost function,  it suffices to record 
\begin{equation*}
s_t = s_{t-1} + w_tx_t x_t^\top, r_t = r_{t-1} + w_t x_t y_t.
\end{equation*}
Then in the $t$-th step, the machine calculates 
\begin{equation*}
 \sum^t_{j = 1}     \frac{w_{t,j}  }{W_t} \nabla \ell_j(  \hat \beta_{t-1})   =  \frac{1}{W_t} \Big ( s_t \hat \beta_{t-1} - r_t \Big ). 
\end{equation*}

\noindent{\bf $t$-dependent weights}: When $w_{t,j}$'s are $t$-dependent such that $w_{t,j}=1/t$, we also only need $\cO(p t_0  +p^2)$ memory  space to record the historical data. For any given $\beta$, to evaluate the cost function,  it suffices to record
\begin{equation*}
s_t = s_{t-1} + x_t x_t^\top, r_t = r_{t-1} +  x_t y_t.
\end{equation*}
Then in the $t$-th step,  the algorithm calculates
\begin{equation*}
 \sum^t_{j = 1}     \frac{w_{t,j}  }{W_t} \nabla \ell_j(  \hat \beta_{t-1})   =  \frac{1}{t} \Big ( s_t \hat \beta_{t-1} - r_t \Big ). 
\end{equation*}

}

In both cases, the  total memory space needed for the $(t+1)$-step is $\cO(t_0 p+p^2)$, independent of $t$.  

\section{Theory}
\label{sec:theory}
After introducing our OLin-LASSO Algorithm \ref{alg:1}, we are now ready to study the statistical properties of the solution sequence $\{ \hat \beta_t\}$ generated by our algorithm.  

To facilitate the analysis, we first introduce a restricted cone around the sparse underlying coefficient $\beta^*$. Specifically, for any set of entry indices $\cA \subset [p]$, we define the following restricted cone 
\begin{equation}\label{def:cone}
\cC_\cA: =  \Big \{ \xi \in \R^p: \|\xi_{\cA^c} \|_1 \leq  3 \| \xi_\cA\|_1 \Big \},
\end{equation} 
where $\xi_\cA$ is a truncation operator that keeps the entries within $\cA$ and truncates the rest  to zero. That is, $[\xi_{\cA}]_i = \xi_i$ if $i \in \cA$ and  $[\xi_{\cA}]_i = 0$ otherwise. To analyze the sparsity property of our estimator $\hat \beta_t$, we utilize the above restricted cone and show that when the regularization coefficient $\lambda_t$ is properly chosen, $(\hat \beta_t - \beta^*)$ falls within the restricted cone induced by the true support $S$ as follows. 
\begin{lemma}\label{lemma:cone}
 If $\lambda_t \geq 2 \Vert \nabla {L}_t(\beta^\ast ; \hat \beta_{t-1})\Vert_\infty$, then $\hat{\beta}_t - \beta^\ast \in \mathcal{C}_S$, where $S = \text{support}(\beta^\ast)$ and $\mathcal{C}_S$ is the cone defined in~\eqref{def:cone}. 
\end{lemma}
We defer the detailed proof to Section \ref{sec:proof_of_lemma_cone} of the supplement. 

The above result is deterministic and holds regardless of the distribution of the covariate vector $x_i$ and response $y_i$. It implies that by choosing a sufficiently large regularizer $\lambda_t$, problem~\eqref{def:beta_t} generates an approximately sparse estimator such that $\hat \beta_t -\beta^*$   falls within the restricted cone $\cC_S$.  

We set $\lambda_t \propto 2 \| \nabla L_t(\beta^* ; \hat \beta_{t-1})\|_\infty$ in the rest of this paper.
To further study the performance of our estimators, we assume the squared loss induced by the initial batch satisfies the following RSC condition. 
\begin{defi}\label{ass:RSC}
A function $f(\beta)$ is said to satisfy the RSC condition if there exists $\kappa >0$ such that for any $\Delta \in \cC_S$ 
\begin{equation*}
f (\beta^* + \Delta) - f (\beta^*) - \Delta^\top \nabla f (\beta^*) \geq \kappa \| \Delta\|_2^2,
\end{equation*}
where $S = \textnormal{support}(\beta^\ast)$ and $\mathcal{C}_S$ is the cone defined in \eqref{def:cone}.
\end{defi} 
\cite{raskutti2010restricted} have shown that the above RSC condition holds provided $n\geq C s\log(p)$ under sub-Gaussian designs. Furthermore, as our loss function $L_t(\bullet;\tilde \beta)$ consists of a squared loss $\ell_0(\beta)$ and a linearized term, if $\ell_0(\beta)$ satisfies the RSC condition, then the RSC condition also holds for all $L_t(\bullet;\tilde \beta)$, which is formally stated as follows. 
\begin{prp}\label{prop:obj_RSC}
Suppose $\ell_0(\beta)$ satisfies the RSC condition with parameter $\kappa$, then for  any $t \geq 1$, $\nabla {L}_t(\bullet ; \tilde \beta )$ also satisifies the RSC condition with parameter $\kappa$. 
\end{prp}

We emphasize  that under our framework, the RSC condition holds for \emph{all} online learning rounds once it holds for the initial batch. As a result, our framework only requires to verify the RSC condition \emph{once} in the initialization phase. In contrast, naively solving a standard  LASSO-based problem in each online round would  require to verify the RSC condition in all rounds, which is quite stringent. This  restricts its applicability to real-world problems.  

The RSC property of our loss function allows us to further understand the behaviors of estimators $\hat \beta_t$. 
In what follows, we build up an upper bound for $\| \hat \beta_t - \beta^*\|_1$ in terms of the regularization cofficient $\lambda_t$ and $\ell_{\infty}$-norm $\Vert \nabla {L}_t(\beta^\ast ; \hat \beta_{t-1}) \Vert_\infty$ of the gradient evaluated at the underlying coefficient~$\beta^*$.
\begin{lemma}\label{lm:bound_2}
		Suppose Assumption~\ref{ass:RSC} holds
and $\lambda_t \geq 2\Vert \nabla {L}_t(\beta^\ast ; \hat \beta_{t-1}) \Vert_\infty$, then 
	\begin{equation*}
\begin{split}
	\Vert \hat{\beta}_t - \beta^\ast\Vert_1 \leq \frac{16s}{\kappa} \big (\lambda_t + \Vert  \nabla {L}_t(\beta^\ast ; \hat \beta_{t-1}) \Vert_\infty \big ) \leq \frac{24s}{\kappa} \lambda_t
 \\
 \text{ and }\| \hat \beta_t - \beta^*\|_2^2  \leq \frac{ \lambda_t  + \| \ \nabla L_t(\beta^\ast;\hat \beta_{t-1}) \| _\infty  }{\kappa}\| \hat \beta_t - \beta^*\|_1.
 \end{split}
	\end{equation*}
\end{lemma}
We defer the detailed proof to Section~\ref{sec:proof_of_lemma_bound2} of the supplement. 


To further quantify the above statistical error, it suffices to provide a statistical bound for $\Vert \nabla {L}_t(\beta^\ast ; \hat \beta_{t-1})\Vert_\infty$. We start with writing
\begin{equation}\label{eq:grad_norm}
\nabla {L}_t(\beta^\ast ; \hat \beta_{t-1})  = \sum^t_{j = 1}   \frac{w_{t,j}}{W_t}  \nabla \ell_j(\beta^\ast) + \left(\hat{\Lambda}_t - \hat{\Lambda}_0 \right)(\hat{\beta}_{t-1} - \beta^\ast),
\end{equation}
where $\hat \Lambda_t = \frac{1}{W_t}\sum^t_{j=1} w_{t,j} x_j x_j $ and $\hat \Lambda_0 = \frac{1}{t_0} \sum^{t_0}_{i=1} x_i x_i$ represent the weighted average of $x_ix_i^\top$'s collected in the online phase and initial batch, respectively.

Letting $\Lambda = \E[x_i x_i^\top]$, we decompose $\hat \Lambda_t - \hat \Lambda_0$ as 
\begin{align*}
    \hat \Lambda_t - \hat \Lambda_0 = \sum^t_{j=1}  \frac{w_{t,j}}{W_t}(x_j x_j - \Lambda ) +  \frac{1}{t_0} \sum^{t_0}_{i=1} ( \Lambda - x_i x_i ). 
\end{align*}
With such a decomposition, each entry of $\frac{w_{t,j}}{W_t}(x_j x_j - \Lambda)$ is mean-zero so that their sum can be viewed as a martingale. Meanwhile, under mild tail assumptions of the data, both of the above terms concentrate with $t_0$ and the online learning rounds $t$ increasing. 

For now, we impose the following assumption on the covariate $x_i$ and noise $\epsilon_i$. 
\begin{assumption}\label{ass:covariate}
Each covariate $x_i$ is independently and identically  distributed such that $x_i -\mu  \sim \text{sub-Gaussian}(\|\Sigma\|_2^2)$ where $\mu = \E x_i \in \R^p$ and  $\Sigma \in \R^{p \times p}$ is the covariance matrix of $x_i$. Each noise $\epsilon_i$ is independently generated and follows a sub-Gaussian distribution such that  $\epsilon_i \sim \text{sub-Gaussian} (\sigma_\epsilon^2)$ for some $\sigma_\epsilon>0$.
\end{assumption}

The above sub-Gaussian assumption is mild and widely adopted by the high-dimensional statistics literature; see for example \cite{raskutti2010restricted, wainwright2019high}. 

To analyze the concentration property of $(\hat{\Lambda}_t - \hat{\Lambda}_0 )$, one major difficulty is that the nonasymptotic upper bound must hold uniformly over all rounds $j=1, 2, \cdots, t$, so we can utilize this to  derive an upper error bound  for $\Vert \hat{\beta}_t - \beta^\ast\Vert_1$ that holds uniformly for the entire solution trajectory.

In order to do so, we need the following lemma for bounding the probability of an upper tail of a sub-martingale, whose proof is provided in Section~\ref{sec:proof_of_lemma_martingale}. 
\begin{lemma}\label{lemma:martingale}
	Assume that $(X_k)$ is a sequence of independent sub-Gaussian random variables, with each $X_k$ having the sub-Gaussian norm given by $\sigma_k$ such that $\E[\exp(\nu X_k)] \leq \exp(\frac{\sigma_k^2 \nu^2}{2})$ for all $\nu >0$. If we define $S_n = \sum^n_{k = 1} X_k$, then 
	\begin{equation*}
	\mathbb{P}(\max_{1 \leq i \leq n} S_i \geq t ) \leq \exp\left(-\frac{t^2}{2\sum^n_{i = 1}\sigma^2_i}\right).
	\end{equation*}
\end{lemma}
In the rest of this paper, we write $z_i = \frac{\sqrt{\sum^j_{k=1} w_k^2}}{\sum^j_{k=1} w_k}  $ for notational convenience. The following result characterizes the concentration properties of $(\hat \Lambda_j - \Lambda)$ and $\sum^t_{j = 1}   \frac{w_{t,j}}{W_t}  \nabla \ell_j(\beta^\ast)$.

\begin{prp}\label{prop:simu_martingale} 
Suppose the weights are 
	\begin{equation*}
	w_{t, j} = \frac{1}{j^a },  
 z_j =  \frac{\sqrt{\sum^j_{k=1} w_k^2}}{\sum^j_{k=1} w_k}  .
	\end{equation*}
 for some $0 \leq a <1$. 
	For any $\epsilon>0$, there exist a constant $c> 0 $ such that with probability at least $1 - 4\delta$, for all $j = 1, 2, \cdots, t$, 
 \begin{equation*}
 \begin{split}
  \quad \Vert \hat{\Lambda}_0 - \Lambda\Vert_{\max} & \leq c   t_0^{-1/2} \sqrt{\log(p^2 / \delta)},
  \\
 \quad \Vert \hat{\Lambda}_j - \Lambda\Vert_{\max} & \leq c  z_j   \sqrt{\log(p^2 / \delta)},
 \\
 	\Vert \sum^j_{k = 1} \frac{w_{j, k}}{W_j}\nabla \ell_k(\beta^\ast) \Vert_\infty & \leq c    z_j   \sqrt{\log(pt/\delta)} \sqrt{\log (p / \delta)}.
  \end{split}
 \end{equation*}
\end{prp}
By applying the above uniform martingale bounds to  the $\ell_\infty$ of $\nabla L_t(\beta^*;\hat \beta_{t-1})$ in \eqref{eq:grad_norm}, we obtain the following result. 
\begin{prp}\label{prp:bound_3}
	For all $\epsilon >0$, there exists a universal constant $c_\epsilon >0$ such that with probability at least $1 - 4\delta$, for all $j \leq t$, we have
\begin{equation*}
\Vert \nabla {L}_j(\beta^\ast;\hat \beta_{j-1} )\Vert_\infty \leq  c z_j \sqrt{\log (p t/ \delta)}   \sqrt{\log (p / \delta)}  +  2c   \sqrt{\log (p^2 / \delta)} \Big ( z_j + \frac{1 }{t_0^{1/2}}
\Big ) \| \hat{\beta}_{j-1} - \beta^\ast  \|_1.
\end{equation*}
\end{prp}
Now, we consider the scenario where the weights $w_{t,k}$ are $t$-independent.  Thus from now on, we sometimes write $w_{t,k}$ as $w_k$. Recall the definition of $z_t$ in Proposition \ref{prop:simu_martingale}.  We first quantify $z_t$ as follows. 
\begin{crl}\label{cor:weights}
Suppose the weight are $t$-independent and set in the form that  $w_{t,j} = \frac{1}{j^a}$ for $0\leq a < 1$, then there exists a constant $c_0>0$ such that for all $j\leq t$,
\begin{enumerate}
    \item[(i)]  $z_j     \leq \frac{c_0}{\sqrt{j}}$ if $0 \leq a < 1/2$;
    \item[(ii)]  $z_j    \leq\frac{c_0 \ln j}{\sqrt{j}}$ if $a= 1/2$;
    \item[(iii)] $z_j   \leq \frac{c_0}{j^{1-a}}$ if $\frac{1}{2} < a <1$.
\end{enumerate}
\end{crl}

Here we note that the decay rate of $ \Vert \nabla {L}_t(\beta^\ast;\hat \beta_{t-1} )\Vert_\infty $ is determined by the diminishing rate of $z_t$.  Corollary \ref{cor:weights} indicates that $z_t$ achieves the fastest diminishing rate when $ 0 \leq a < 1/2$. Thus, in the rest of our paper, we adopt this choice of weights so that $z_j \lesssim j^{-1/2}$. Next, by combining Proposition \ref{prp:bound_3} with Lemma~\ref{lm:bound_2} and setting $\lambda_j \propto 2 \| \nabla L_j(\beta; \hat \beta_{j-1})\|_\infty$, for all $j \leq t$, we can bound the $\ell_1$ error as 
\begin{equation*}
\Vert \hat{\beta}_j - \beta^\ast\Vert_1  \leq  \frac{ 48sc}{\kappa}\sqrt{ \frac{ \log (p t / \delta) }{j}} \log (p / \delta) + \frac{ 96 s c }{\kappa}\Big ( \frac{c_0}{j^{1/2}} +  \frac{1}{t_0^{1/2}}\Big ) \Vert \hat{\beta}_{j-1} - \beta^\ast\Vert_1.
\end{equation*}
The above result suggests that with high probability, for each online learning round $j$, the statistical error $\Vert \hat{\beta}_j - \beta^\ast\Vert_1 $ can be bounded in terms of the statistical error $ \Vert \hat{\beta}_{j-1} - \beta^\ast\Vert_1 $ incurred within the previous online learning round plus an $\cO(\sqrt{ \frac{\log t}{j}})$ term, which serves as the building block in analyzing the behavior of $\{ \hat \beta_j \}_{j=1}^t$.

Finally, because the above relationship holds for all $j\leq t$, by recursively applying it,  we characterize the convergence behavior of $\| \hat \beta_t - \beta^*\|_1$ as follows, and provide its proof in Section~\ref{sec:proof_of_thm_main}. 
\begin{thm}\label{thm:main}
Suppose Assumption \ref{ass:covariate} holds and $\ell_0(\beta)$ satisifies the RSC condition \eqref{ass:RSC}. 
	With probability at least $1 - 4 \delta$, we have 
\$\Vert \hat{\beta}_t - \beta^\ast\Vert_1 \leq \sum^t_{j = 1} \left( \prod^t_{k = j + 1} a_k\right) b_j  +  \left(\prod^t_{j = 1} a_j\right)    \Vert \hat{\beta}_0 - \beta^\ast\Vert_1
\$
where 
\begin{align*}
a_j & = \frac{ 96 s c}{\kappa} \sqrt{\log(p^2/\delta)} \left( \frac{c_0}{j^{1/2}}+  \frac{1}{t_0^{1/2}}\right),
\\
b_j & = \frac{ 48sc}{\kappa}\sqrt{ \frac{ \log (pt / \delta) }{j}} \sqrt{ \log (p / \delta) }.
\end{align*}
In addition, suppose the size of initial batch $t_0$ ensures $a_j < 1$ for large $j$, we have with probability at least $1-4\delta$,
    \begin{equation*}
        \begin{split}
       \Vert \hat{\beta}_t - \beta^\ast\Vert_1 
       \leq  \widetilde \cO \Big (\frac{s }{\kappa \sqrt{t}} \Big  )  \text{ and }
       \Vert \hat{\beta}_t - \beta^\ast\Vert_2 
       \leq  \widetilde \cO \Big ( \frac{\sqrt{s} }{\kappa \sqrt{t}} \Big  ). 
        \end{split}
    \end{equation*}
\end{thm}
\begin{remark}
It is worth pointing out that the convergence of our algorithm is ensured if the size of initial batch $t_0 \geq (96 s c /\kappa)^2 \log(p^2/\delta)$, which only depends on the distribution of covariate $x_j$'s, the condition number $\kappa$, and the sparsity level $s = \|\beta^*\|_0$, instead of the underlying parameter $\beta^*$. In comparison with the minimal signal assumption $\min_{j \in S} \|\beta_j^*\| \geq 2c \sqrt{\log p/t_0}$ by \cite{fan2018statistical}, which guarantees the support of $\beta^*$ can  be identified by the initial batch with $t_0$ data points, our approach  completely removes this minimal signal assumption. Consequently, our approach would exhibit superior performance in weak signal scenarios where $\min_{j \in S} \|\beta_j^*\|$ is small or  the initial batch is small. We verify this phenonmenon in our numerical experiments. 
\end{remark}
\begin{remark} Notably, as discussed in Proposition~\ref{prop:obj_RSC}, all loss functions $\{ L_j(\bullet;\tilde \beta) \}_{j=0}^t$ satisfy  the RSC condition once it holds for the squared loss $\ell_0(\beta)$ induced by the initial batch. Therefore, when employing our framework, it suffices to guarantee that the RSC condition \emph{once} for the initial batch instead of guaranteeing  it for all online rounds. This boosts the applicability of our method in practice especially when there exist  badly behaved data points in certain rounds. 
\end{remark}
\begin{remark}
Theorem~\ref{thm:main} quantifies the behavior of the entire solution trajectory $\{\hat \beta_j\}_{j=1}^t$ and ensures its good performance over all online rounds. 
This implies that our approach generates better and better estimators as the online learning round increases. The solution sequence converges to the underlying coefficient vector $\beta^*$ at the rate of $\widetilde \cO( \sqrt{s/t})$  in terms of the $\ell_2$-norm error, which matches the minimax lower bound of the sparse linear regression in the offline case up to logarithmic terms \citep{raskutti2011minimax}. Compared with existing work on offline $\ell_2$-regularized regression that requires to store the entire dataset of cost $\cO(pt)$, such as the offline LASSO, our approach enjoys a much lower \emph{fixed} memory cost of $\cO( p t_0 + p^2)$ and is favored for high-dimension big-data applications where both covariate dimension $p$ and number of data points $t$ are large. 
\end{remark}

\section{Numerical Experiments}
\label{sec:experiments}

After studying the theoretical performance of our  {\tt Olin\_LASSO} algorithm in Section~\ref{sec:theory}, we now continue to investigate the practical performance of our algorithm in various settings. We consider the scenario where the covariate vector is of high dimensionality such that $p = 1000$. 
We consider the \emph{weak} signal scenario and generate the underlying coefficient $\beta^*$ as follows:
We generate a sparse vector whose first 20 entries are nonzero and the rest are zeros. In this case, we let $S = \text{Support}(\beta^*) = \{1,2,\cdots,20\}$ be the support of underlying $\beta^*$ and have the sparsity level $s = |S| = \| \beta^*\|_0 = 20$. For each entry $j \in S$, we independently generate a weak signal such that $\beta_j^* \sim \cN(0,0.25)$. In our experiments, the mean absolute value of the these nonzero entries is $0.261$, and the $\ell_1$- and $\ell_2$-norm of $\beta^*$ are $\| \beta^*\|_1 = 5.22$ and $\| \beta^*\|_2 = 1.30$, respectively.

We investigate the numerical performance of our algorithm against a baseline algorithm {\tt OS\_LASSO\_K} \citep{fan2018statistical}. {\tt OS\_LASSO\_K} is a two-phase algorithm that adopts the standard least-squared loss $h_t(\beta)$ instead of our variant. In the burn-in phase,  {\tt OS\_LASSO\_K} determines the set of nonzero entries $S_0$ of $\beta^*$ by running a penalized regression using the initial batch of size $t_0$. We will use LASSO in all of our experiments. In the online learning phase, upon receiving each online data point $(x_i,y_t)$, {\tt OS\_LASSO\_K}  conducts $K$ iterative hard-thresholding 
gradient descent steps $\beta_{i,k+1} = \Pi_{S_0}(\beta_{i,k}  - \eta \nabla h_t(\beta_{i,k}))$ for $k=1,\cdots,K$
that only keeps the entries within $S_0$ and truncates the rest entries to zero.

We test our algorithm and the baseline algorithms over the covariate correlation setup:
\begin{itemize}
    \item Toeplitz $\rho = 0.5$ correlation: We generate the covariate $x_i \in \R^{p}$ under a multivariate norm distribution that $x_i \sim \cN({\bf{0}},\Sigma)$ where $\Sigma \in \R^{p \times p }$ is a covariance matrix such that $\Sigma_{i,i} = 1$ and $\Sigma_{i,j} = \rho^{|i-j|}$ for all $i\neq j$. 
\end{itemize}
In the above setup, to generate data pair $(x_i,y_i)$, we first generate the covariate $x_i \in \R^p$ as above, then independently generate a noise term $\epsilon_i \sim \cN(0,1)$, and set the response as $y_i = x_i^\top \beta^* + \epsilon_i$. 

We consider following two experiments:
\begin{enumerate}
    \item[1.] We test the performance of above algorithms over different initial batch sizes $t_0 \in \{ 50,100,$ \\ $150,\cdots ,500\}$ under the above Toeplitz $\rho = 0.5$ correlation design. After running each simulation for $T={10}^4$ online learning  rounds, we plot the obtained mean squared error (MSE) $\| \beta_T - \beta^* \| ^2_2$ against the  initial batch size $t_0$ in Figure~\ref{fig:init}. We set the regularization coefficient for our {\tt Olin\_LASSO} algorithm as  $\lambda_t = \sqrt{ \frac{\log(p)}{t} }$, and set the step-size of the baseline algorithms {\tt OS\_LASSO\_K=1} and {\tt OS\_LASSO\_K=20} as $\eta = 0.001$.

    \item[2.] We fix the initial batch size as $t_0 = 100$, test the algorithms over $T=10^4$ online rounds under the above covariate design, and plot the MSE $\| \beta_t - \beta^* \| ^2_2$ against online round $t$ in Figure \ref{fig:online}. Other parameters are set as the same as the first experiment.
\end{enumerate}

\begin{figure}[t]
  \centering
  \includegraphics[width=0.6\textwidth]{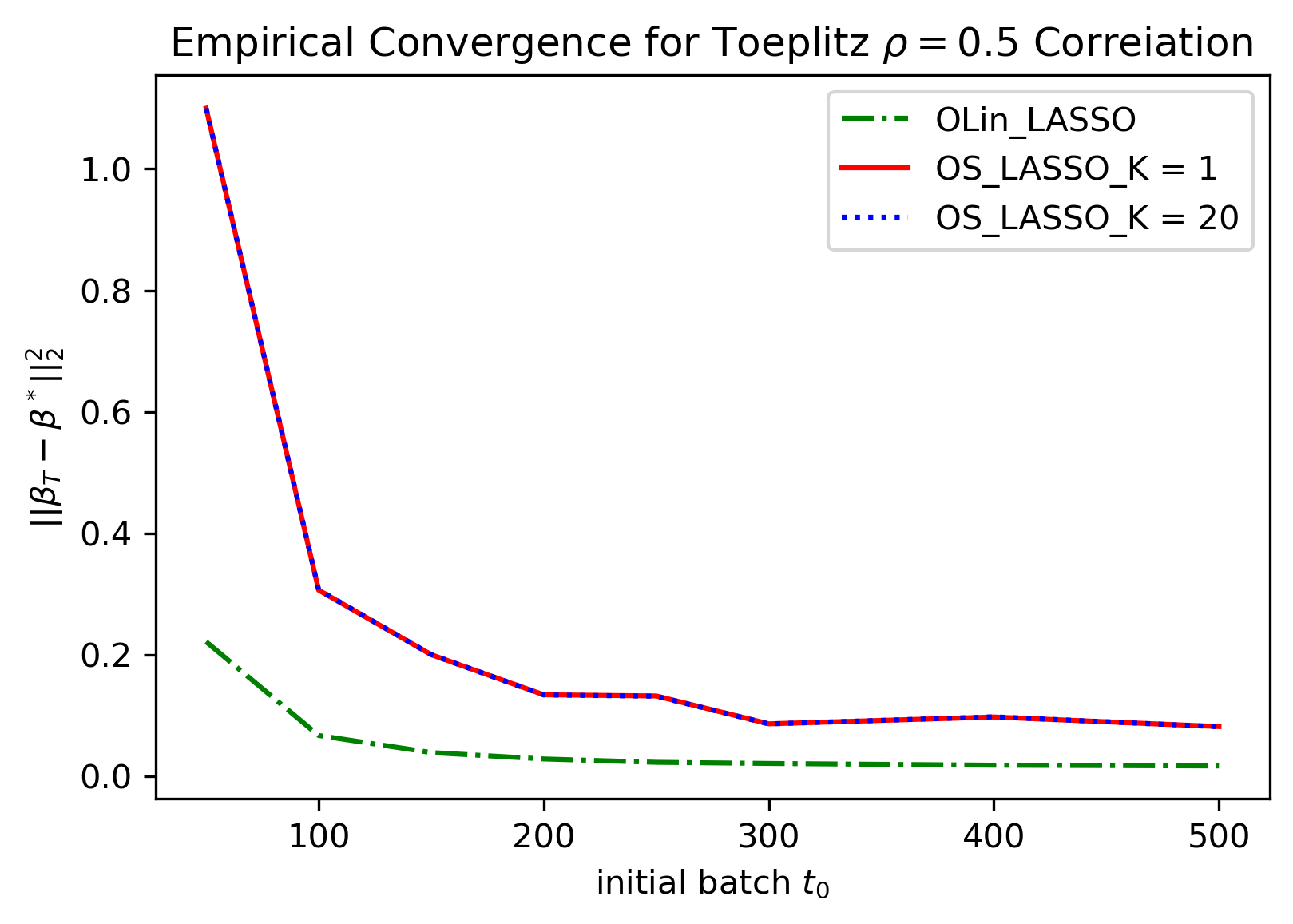}
  \caption{Empirical Convergence of MSE $\| \beta_T - \beta^* \| ^2_2$ under weak signal setup for $t_0 =50,100,150,\cdots,500$ and $T=10000$ online learning rounds for Toeplitz $\rho=0.5$ covariate correlation design.}
  \label{fig:init}
    \vskip -0.31cm
\end{figure}
\begin{figure}[t]
  \centering
  \includegraphics[width=0.6\textwidth]{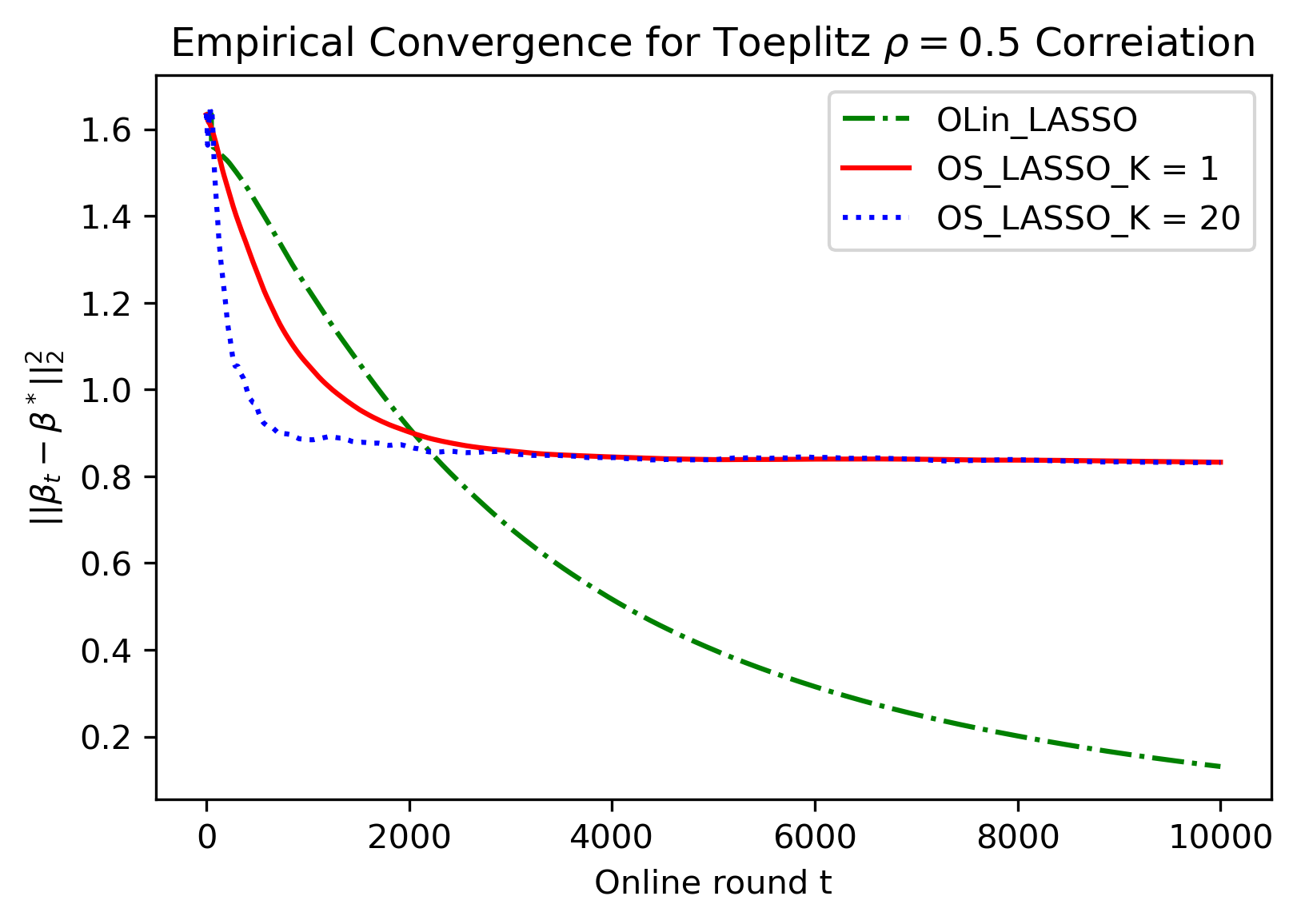}
  \caption{Empirical Convergence of MSE $\| \beta_t - \beta^* \| ^2_2$ under weak signal setup for online learning rounds $t \in [0,10000]$ and initial batch size $t_0 = 100$  for Toeplitz $\rho=0.5$ covariate correlation design.}
  \vskip -0.31cm
  \label{fig:online}
\end{figure}
From Figure~\ref{fig:init}, we observe that our algorithm outperforms {\tt OS\_LASSO\_K=1}  and {\tt OS\_LASSO\_K=20}, especially when the initial batch size $t_0$ is small. For example, when $t_0 = 100$,  {\tt OS\_LASSO\_K=1}  and {\tt OS\_LASSO\_K=20} generates estimator of MSE being around $0.49$, 
 while our {\tt Olin\_LASSO} algorithm generate a better estimator that has a MSE around $0.05$. 


Further, from Figure~\ref{fig:online}, we can see that {\tt OS\_LASSO\_K=1}  and {\tt OS\_LASSO\_K=20} do not generate sequences that converge to the underlying estimator $\beta^*$, as evidenced by the observation that the MSE is not improving when receiving more online data. In contrast, our algorithm can generate consistent estimators that converge to the underlying $\beta^*$. 

Both of the above experiments suggest that {\tt OS\_LASSO} is sensitive to the size of initial batch $t_0$ and can only converge when $t_0$ is large enough to determine the support of $\beta^*$, which requires a large initial batch in the weak signal scenario. By contrast, our algorithm exhibits superior performance to {\tt OS\_LASSO} even if  only a small initial batch is available. 

\begin{figure}[t]
  \centering
  \includegraphics[width=0.6\textwidth]{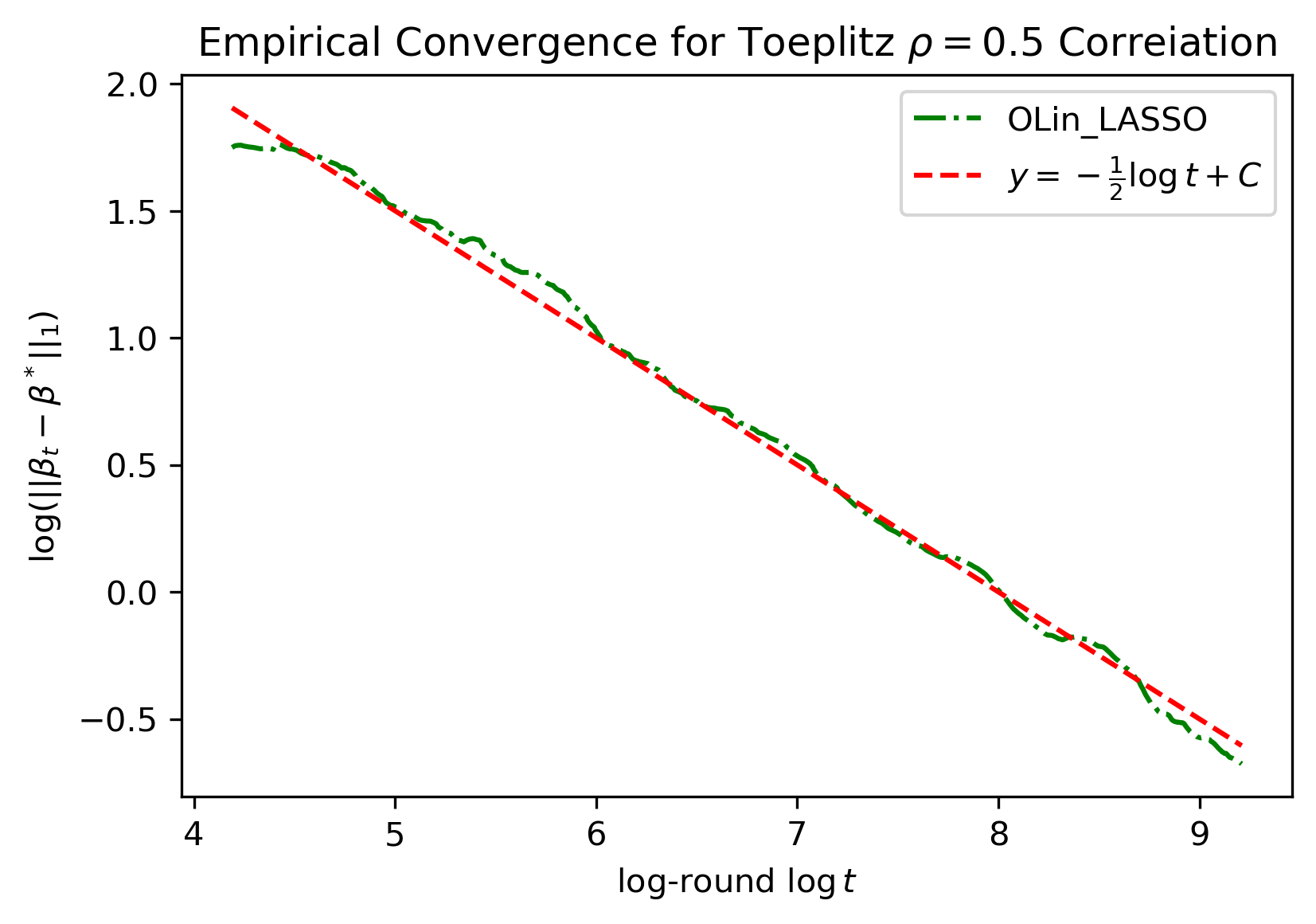}
  \caption{$\log(\| \beta_t - \beta^* \|_1$) against $\log(t)$ under weak signal setup for $t_0 = 500$ and $T=10000$ online learning rounds for Toeplitz $\rho=0.5$ covariate correlation design.}
  \label{fig:log_rate}
  \vskip -0.3cm
\end{figure}

To verify the rate of convergence provided by our theorem,  we plot the log-error $\log( \| \beta_t - \beta^*\|_1)$ against the log-round $\log t$ in Figure~\ref{fig:log_rate}. We observe that it preserves a slope around $-1/2$. This suggests that $\{\hat \beta_t\}$ indeed converges to $\beta^*$ at the rate of $\tilde \cO(\frac{1}{\sqrt{T} })$ under $\ell_1$-norm, which matches Theorem~\ref{thm:main}. 

In addition, we test our algorithm against the baselines under other types of covariate correlation distributions. We also test these algorithms under a  strong signal setting where $\beta^*$  has significant absolute values for nonzero entries. Our algorithm exhibits superior performance against the baseline algorithms in all scenarios. We provide detailed experiment setups and numerical results in Section~\ref{sec:more_numerics} of the supplement.

\section{Conclusion}
In this paper, we propose a novel framework to solve the online sparse linear regression problem. Our framework adopts a loss function that consists of a squared loss induced by the initial batch  and a linearized term induced by the online  data, and our approach is memory efficient with a memory complexity of $\cO( pt_0 + p^2)$. By ensuring the RSC condition only for the initial squared loss, the entire solution trajactory $\{ \hat \beta_t\}$  falls into a restricted set, and converges to $\beta^*$ in the optimal rate  of $\widetilde \cO(\sqrt{s/t} )$ 
under $\ell_2$-norm, establishing the benchmark in online sparse linear regression. 

In addition, our algorithm exhibits superior practical performance comparing with the existing arts under various design correlations. These numerical experiments also aligns with our nonasymptotic rate of convergence result well. 

One possible limitation of our framework is that it requires an initial batch of size $t_0 \geq (96 s c /\kappa)^2 \log(p^2/\delta)$ to ensure convergence, where the primitives such as the sparsity level $s$ and the conditional number $\kappa$ are unknown in practice. It remains an open problem whether problem-independent algorithms can be developed to avoid the usage of such primitives but still exhibit a comparable performance theoretically.  

\bibliographystyle{ims}
\bibliography{online}

\appendix
\renewcommand{\theequation}{S.\arabic{equation}}
\renewcommand{\thetable}{S.\arabic{table}}
\renewcommand{\thefigure}{S.\arabic{figure}}
\renewcommand{\thesection}{S.\arabic{section}}
\renewcommand{\thelemma}{S.\arabic{lemma}}

\vspace{30pt}
\noindent{\bf \LARGE Appendix}

\section{Proof of Results}
\subsection{Proof of Lemma~\ref{lemma:cone}} \label{sec:proof_of_lemma_cone}
\begin{proof}
	Recall that 
 $$\hat{\beta}_t \in \argmin_{\beta \in \R^p} \{ L_t(\beta;\hat \beta_{t-1}) + \lambda_t \| \beta\|_1 \},
 $$ by using the convexity of   $L_t(\beta;\hat \beta_{t-1})$ in $\beta$, we have 
	\begin{equation*}
	(\hat{\beta}_t - \beta^\ast)^\top  \nabla L_t(\beta^\ast;\hat \beta_{t-1}) \leq  	 L_t(\hat \beta_t;\hat \beta_{t-1}) - L_t(\beta^*;\hat \beta_{t-1}) \leq  \lambda_t(\Vert \beta^\ast\Vert_1 - \Vert \hat{\beta}_t\Vert_1),
	\end{equation*}
which further yields that 
	\begin{equation}\label{eq:le01}
	-\Vert \hat{\beta}_t - \beta^\ast\Vert_1 \Vert \nabla L_t(\beta^\ast;\hat \beta_{t-1}) \Vert_\infty \leq \lambda_t(\Vert \beta^\ast\Vert_1 - \Vert \hat{\beta}_t\Vert_1).
	\end{equation}
	Because $S$ is the support of $\beta^*$, using the inequalities
	\begin{equation*}
	\| \beta_{S^c}\|_1 = \| \beta - \beta^*\|_1  - \| (\beta - \beta^*)_S\|_1
	\end{equation*}
	and 
	\begin{equation*}
	 \| \beta^* \|_1 - \| \beta\|_1 = \| \beta^* \|_1 - \| \beta_S\|_1 - \| \beta_{S^c}\|_1 \leq \| (\beta^* - \beta )_S\|_1 - \| \beta^*_{S^c}\|_1,
  	\end{equation*}
	we obtain 
	\begin{equation*}
\Vert\beta^\ast\Vert_1 - \Vert\hat{\beta}\Vert_1 \leq 2\Vert (\hat{\beta} - \beta^\ast)_S\Vert_1 - \Vert\hat{\beta} - \beta^\ast\Vert_1.
\end{equation*}
By combining the above inequality with \eqref{eq:le01} and rearranging the terms, we conclude that 
\begin{equation*}
\begin{split}
&  \frac{2\lambda_t}{\lambda_t - \Vert \nabla L_t(\beta^\ast;\hat \beta_{t-1}) \Vert_\infty} \Vert (\hat{\beta}_t - \beta^\ast)_S\Vert_1 
\\
& \quad =  \frac{2}{1 - \Vert \nabla L_t(\beta^\ast;\hat \beta_{t-1}) \Vert_\infty/\lambda_t } \Vert (\hat{\beta}_t - \beta^\ast)_S\Vert_1 
\geq \Vert \hat{\beta}_t - \beta^\ast\Vert_1.
\end{split}
\end{equation*}

Our choice of $\lambda_t \geq 2 \Vert \nabla L_t(\beta^\ast;\hat \beta_{t-1}) \Vert_\infty $  makes the constant multiplier in the left-hand side upper bounded by 4, which proves the desired result. 
\end{proof}



\subsection{Proof of Lemma~\ref{lm:bound_2}} \label{sec:proof_of_lemma_bound2} 

\begin{proof}
 Setting $\lambda_t \geq 2 \|  \nabla L_t(\beta^\ast;\hat \beta_{t-1}) \|_\infty$, we have  $\hat \beta_t \in \cC_S$, which together with the RSC Assumption~\ref{ass:RSC} implies that 
	\begin{equation*}
	\begin{split}
	\lambda_t ( \| \beta^*\|_1 - \| \hat \beta_{t}\|_1) & \geq \langle  \nabla L_t(\beta^\ast;\hat \beta_{t-1}) , \hat \beta_t - \beta^* \rangle + \kappa \| \hat \beta_t - \beta^*\|_2^2 \\
	& \geq - \|  \nabla L_t(\beta^\ast;\hat \beta_{t-1}) \| _\infty \| \hat \beta_t - \beta^*\|_1 + \kappa \| \hat \beta_t - \beta^*\|_2^2.
	\end{split}
	\end{equation*}
	By using the fact that $ \| \hat \beta_t - \beta^*\|_1 \geq \| \beta^*\|_1 - \| \hat \beta_{t}\|_1 $ and rearranging the above terms, we further obtain 
	\begin{equation*}
	\| \hat \beta_t - \beta^*\|_2^2  \leq \frac{ \lambda_t  + \| \ \nabla L_t(\beta^\ast;\hat \beta_{t-1}) \| _\infty  }{\kappa}\| \hat \beta_t - \beta^*\|_1.
	\end{equation*}
Because $\hat \beta_t \in \cC_S = \Big \{ \beta \in \R^p: \|(\beta - \beta^*)_{S^c} \|_1 \leq  3 \| (\beta - \beta^*)_S\|_1 \Big \}$, we obtain 
	\begin{equation*}
	\| \hat \beta_t - \beta^* \|_1 \leq 4  \| ( \hat \beta_t - \beta^* )_S\|_1 \leq 4 \sqrt{s}   \| ( \hat \beta_t - \beta^* )_S\|_2,
	\end{equation*}
 which further leads to 
 \begin{equation*}
   \| \hat \beta_t - \beta^* \|_1^2 \leq  16 s  \| ( \hat \beta_t - \beta^* )_S\|_2^2     \leq \frac{ 16s }{\kappa} \Big ( \lambda_t  + \|  \nabla L_t(\beta^\ast;\hat \beta_{t-1}) \| _\infty \Big ) \| \hat \beta_t - \beta^* \|_1. 
 \end{equation*}
We obtain the desired result by dividing both sides by $ \| \hat \beta_t - \beta^* \|_1$ and setting $\lambda_t \geq  2\Vert  \nabla L_t(\beta^\ast;\hat \beta_{t-1}) \Vert_\infty $. 
\end{proof}

\subsection{Proofs of Lemma~\ref{lemma:martingale}, Proposition~\ref{prop:simu_martingale}, and Proposition \ref{prp:bound_3}} \label{sec:proof_of_lemma_martingale}
\begin{proof}[Proof of Lemma~\ref{lemma:martingale}]
On one hand, the Doob's  inequality implies
\begin{equation*}
\mathbb{P}(\max_{1 \leq i \leq n} S_i \geq t) = \mathbb{P}(\max_{1 \leq i \leq n} e^{hS_i} \geq e^{ht}) \leq \mathbb{E}[e^{hS_n - ht}] = e^{-ht}\prod^n_{i = 1}\mathbb{E}[e^{hX_i}],
\end{equation*}
because $e^{hS_i}$ is a sub-martingale. On the other hand, since all the $X_i$'s are sub-Gaussian, we have
\begin{equation*}
\mathbb{E}[e^{hX_i}] \leq e^{h^2\sigma^2_i / 2}. 
\end{equation*}
Then we obtain
$$\mathbb{P}(\max_{1 \leq i \leq n} S_i \geq t) \leq e^{\frac{h^2}{2}\sum\limits_{i=1}^n\sigma_i^2- ht}.$$
We obtain the desired result by setting $h = ( \sum\limits_{i=1}^n\sigma_i^2)^{-1}t$ in the above inequality. 
\end{proof}

\begin{proof}[Proof of Proposition \ref{prop:simu_martingale} ]

	Obviously we have $\nabla \ell_k(\beta^\ast) = 2  x_k \epsilon_k$, which means 
	\begin{equation*}
	S_j = \sum^j_{k=1} w_k \nabla \ell_k(\beta^\ast) = 2\sum^j_{k=1} w_k  x_k \epsilon_k
	\end{equation*}
	is a martingale (for each component). Clearly,  the sub-Gaussian norm of each component of $w_k x_k \epsilon_k$ is $w_k M_t \sigma_\epsilon$, where $M_t = \max_{i \leq t, j \leq p} |x_{ij}|$ and $\sigma_\epsilon$ represents the sub-Gaussian norm of $\epsilon_k$. If we apply  Lemma~\ref{lemma:martingale} and use  a union bound, we obtain that for all $ a >0$, 
	\begin{equation*}
	\mathbb{P}\left(\max_{1 \leq j \leq t} \Vert S_j \Vert_\infty \geq a\right) \leq p\exp\left(-\frac{a^2}{2\sigma_\epsilon^2 M_t^2 \sum^t_{j=1} w^2_j}\right).
	\end{equation*}
	This means that with probability at least $1 - \delta$, for all $j = 1, 2, \cdots, t$, we have
	\begin{equation*}
	\Vert 	S_j\Vert_\infty \leq \sigma_\epsilon M_t \sqrt{\sum^j_{k=1} w_k^2} \sqrt{\log (p / \delta)}.
	\end{equation*}
	Dividing both sides by $\sum^j_{k=1} w_k$, we obtain that
	\begin{equation*}
	\Vert 	\frac{S_j}{W_j}\Vert_\infty \leq \sigma_\epsilon M_t \frac{\sqrt{\sum^j_{k=1} w_k^2}}{\sum^j_{k=1} w_k} \sqrt{\log (p / \delta)}=\sigma_\epsilon M_t z_j \sqrt{\log (p / \delta)}.
	\end{equation*}
 Because each $x_{i,j}$ is a sub-Gaussian random variable, letting $\sigma_X$ be its corresponding sub-Gaussian norm and applying the union bound, we have that for any $a>0$,
 \begin{equation*}
     \mathbb{P} \left (M_t \geq a \right ) \leq 2 p^2 \exp\left (-\frac{a^2}{2\sigma_X^2} \right ). 
 \end{equation*}
 Combine the above bounds, we have with probability at least $1-2\delta$, there exists a constant $c_1>0$ such that
	\begin{equation*}
	\Vert 	\frac{S_j}{W_j}\Vert_\infty \leq   c_1 z_t \sqrt{\log (p / \delta)}\sqrt{\log (pt / \delta)}  .
	\end{equation*}
 This proves the first inequality. For the second and third inequalities, note that with probability at least $1-2\delta$, there exists a constant $c_2>0$ satisfying that:  
	\begin{equation*}
 \begin{split}
 \Vert \hat{\Lambda}_j - \Lambda \Vert_{\infty} & = \Vert \sum^j_{k=1} \frac{w_{j, k} }{W_j}  (x_k x_k^\top  - \Lambda)\Vert_{\infty}
 \leq c_2 z_j \sqrt{\log (p^2 / \delta)},
 \\
 \text{ and }
 \Vert \hat{\Lambda}_0 - \Lambda \Vert_{\infty} 
 & 
 = \Vert \sum_{j=1}^{t_0} \frac{1 }{t_0}  (x_{0j} x_{0j}^\top  - \Lambda)\Vert_{\infty} 
 \leq c_2 t_0^{-\frac{1}{2}}\sqrt{\log (p^2 / \delta)}.
 \end{split}
	\end{equation*}
We complete the proof by setting $c  = \max \{ c_1, c_2\}$.
\end{proof}


\begin{proof}[Proof of Proposition \ref{prp:bound_3}]
We first write $\nabla {L}_t(\beta^\ast ; \hat \beta_{t-1}) $ as
\begin{equation*}
\begin{aligned}
\nabla {L}_t(\beta^\ast ; \hat \beta_{t-1})
&= \frac{1}{W_t}\sum^t_{j = 1} w_{t,j}( \nabla \ell_j(\hat{\beta}_{t-1}) -  \nabla \ell_j(\beta^\ast)) + \frac{1}{W_t}\sum^t_{j = 1} w_{t,j}  \nabla \ell_j(\beta^\ast) + (\nabla \ell_0(\beta^\ast) - \nabla \ell_0(\hat{\beta}_{t-1}))\\
&= \frac{1}{W_t}\sum^t_{j = 1} w_{t,j}  \nabla \ell_j(\beta^\ast) + \left(\frac{1}{W_t}\sum^t_{j=1} w_{t,j} x_j x_j - \frac{1}{t_0} \sum^{t_0}_{i=1} x_i x_i \right)(\hat{\beta}_{t-1} - \beta^\ast)\\
&= \frac{1}{W_t}\sum^t_{j = 1} w_{t,j}  \nabla \ell_j(\beta^\ast) + \left(\hat{\Lambda}_t - \hat{\Lambda}_0 \right)(\hat{\beta}_{t-1} - \beta^\ast),
\end{aligned}
\end{equation*}
which implies 
\begin{equation*}
\begin{aligned}
\| \nabla {L}_t(\beta^\ast ; \hat \beta_{t-1})  \|_\infty & \leq \Big \| \sum^t_{j = 1} \frac{w_{t,j} }{W_t}  \nabla \ell_j(\beta^\ast) \Big \|_\infty +  \Big \| \left(\hat{\Lambda}_t - \hat{\Lambda}_0 \right)(\hat{\beta}_{t-1} - \beta^\ast) \Big \|_\infty
\\
& \leq  \Big \| \sum^t_{j = 1} \frac{w_{t,j} }{W_t}  \nabla \ell_j(\beta^\ast) \Big \|_\infty +  2\Big ( \| \hat{\Lambda}_t -  \Lambda \|_\infty + \| \hat{\Lambda}_0   - \Lambda  \|_\infty   \Big ) \| \hat{\beta}_{t-1} - \beta^\ast  \|_1.
\end{aligned}
\end{equation*}
	It then suffices to apply Proposition \ref{prop:simu_martingale}. We conclude that with probability at least $1-4\delta$, there exists a constant $c >0$ such that
 \begin{equation*}
\begin{aligned}
& \| \nabla {L}_t(\beta^\ast ; \hat \beta_{t-1})  \|_\infty \leq  c z_t \sqrt{\log (p / \delta)}\sqrt{\log (pt / \delta)}  +  2c   \sqrt{\log (p^2 / \delta)} \Big ( z_t + \frac{1 }{t_0^{1/2}}
\Big ) \| \hat{\beta}_{t-1} - \beta^\ast  \|_1. 
\end{aligned}
\end{equation*}
This completes the proof. 
\end{proof}

\subsection{Proof of Theorem \ref{thm:main}} 
The following lemma can be proved by induction, and we omit its proof.
\begin{lemma}
	Assume that a sequence $\Delta_t$ satisfies the following recursive relationship
	\begin{equation*}
	\Delta_t \leq b_t + a_t \Delta_{t-1},
	\end{equation*}
	then we have
	\begin{equation*}
	\Delta_t \leq \sum^t_{j = 1} \left( \prod^t_{k = j + 1} a_k\right) b_j  +  \left(\prod^t_{j = 1} a_j\right) \Delta_0.
	\end{equation*}
\end{lemma}


\label{sec:proof_of_thm_main}
\begin{proof}[Proof of Theorem \ref{thm:main}]

{\bf (a)}
We combine Lemma \ref{lm:bound_2} and Proposition \ref{prp:bound_3} to prove the desired result. Let $\Delta_j = \|\hat \beta_j - \beta^* \|_1$ and set 
\begin{equation*}
\begin{aligned}
a_j  = \frac{ 96 s c}{\kappa} \sqrt{\log(p^2/\delta)} \left( \frac{c_0}{j^{1/2}}+  \frac{1}{t_0^{1/2}}\right) \text{ and }
b_j  = \frac{ 48sc}{\kappa}\sqrt{ \frac{ \log (pt / \delta) }{j}} \sqrt{ \log (p / \delta) }.
\end{aligned}
\end{equation*}
We need to upper bound 
\begin{equation*}
\sum^t_{j = 1} \left( \prod^t_{k = j + 1} a_k\right) b_j  +  \left(\prod^t_{j = 1} a_j\right) \Delta_0.
\end{equation*}
First note that for $j\geq t_0$,
\begin{equation*}
A_j = \prod^t_{k = j + 1} a_k \leq   \left ( \frac{ 96 s c}{\kappa} \sqrt{\log(p^2/\delta)} \left(  \frac{1+c_0}{t_0^{1/2}}\right) \right )^{t-j} .  
\end{equation*}
As a result,
\begin{equation*}
\sum^t_{j = 1} A_j b_j \leq \frac{ 48sc}{\kappa}\sqrt{  \log (pt / \delta) }\sqrt{ \log (p / \delta) } \sum^t_{j = 1} \left (  \frac{ 96 s c}{t_0^{1/2} \kappa} \sqrt{\log(p^2/\delta)}(1+c_0) \right )^{t-j} \frac{1}{\sqrt{j}} . 
\end{equation*}
Letting $q = \frac{ 96 s c}{t_0^{1/2} \kappa } \sqrt{\log(p^2/\delta)}(1+c_0)<1$ and by Lemma~\ref{lemma:series}, we have
\begin{equation*}
\sum_{j=1}^t q^{t-j} \frac{1}{\sqrt{j}} \leq \frac{C\ln t}{\sqrt{t}}
\end{equation*}
for some $C>0$. 
We hence obtain that with probability at least $1 - 4\delta$, 
\begin{equation*}
\Vert \hat{\beta}_{t} - \beta^\ast\Vert_1 \leq \prod^t_{j = 1} a_j \Delta_0+ \frac{ 48sc \ln{t}}{\kappa}\log(p / \delta)\sqrt{\frac{\log(pt / \delta)}{t}}.
\end{equation*}
Since $\prod^t_{j = 1} a_j=\prod^{t_0}_{j = 1} a_j\prod^t_{j = t_0+1} a_j\leq \delta^{t-t_0}\prod^{t_0}_{j = 1} a_j $, we can set $\prod^{t_0}_{j = 1} a_j \Delta_0 = M$, which is fixed and positive. Then $\prod^t_{j = 1} a_j \Delta_0 \leq \delta^{t-t_0}M$ geometrically decays to $0$. As a result, we conclude that with probability at least $1-4\delta$, 
\begin{equation*}
\begin{split}
\| \hat \beta_t - \beta^*\|_1  \leq \widetilde \cO \Big (  \frac{\sqrt{s}}{\kappa \sqrt{t}} \Big ).
\end{split}
\end{equation*}
\noindent
{\bf (b)}
To derive the statistical error in terms of $\ell_2$-norm, we apply  Lemma~\ref{lm:bound_2} and obtain
\begin{equation}\label{eq:2normerror}
\begin{split}
\| \hat \beta_t - \beta^*\|_2^2  \leq \frac{ \lambda_t  + \| \ \nabla L_t(\beta^\ast;\hat \beta_{t-1}) \| _\infty  }{\kappa}\| \hat \beta_t - \beta^*\|_1
 \leq \frac{ 3\| \ \nabla L_t(\beta^\ast;\hat \beta_{t-1}) \| _\infty  }{\kappa} \| \hat \beta_t - \beta^*\|_1.
\end{split}
\end{equation}
If suffices to bound the term $\| \nabla L_t(\beta^*;\hat \beta_{t-1}) \|_\infty$. To do so, we combine Lemma~\ref{lm:bound_2} and Proposition~\ref{prp:bound_3} and obtain with probability at least $1-4\delta$ that  for all $j \leq t$
\begin{equation*}
\begin{split}
\Vert \nabla {L}_j(\beta^\ast;\hat \beta_{j-1} )\Vert_\infty & \leq   c z_j \sqrt{\log (p t/ \delta)}\sqrt{\log (p / \delta)} + 2c   \sqrt{\log (p^2 / \delta)} \Big ( z_j +\frac{1 }{t_0^{1/2}}
\Big ) \| \hat{\beta}_{j-1} - \beta^\ast  \|_1\\
& \leq c z_j \sqrt{\log (p t/ \delta)}\sqrt{\log (p / \delta)} + 2c   \sqrt{\log (p^2 / \delta)} \Big ( z_j +\frac{1 }{t_0^{1/2}}
\Big )\frac{16s}{\kappa} \Vert \nabla {L}_{j-1}(\beta^\ast;\hat \beta_{j-2} )\Vert_\infty.
\end{split}
\end{equation*}
The condition $q = \frac{ 96 s c}{t_0^{1/2} \kappa } \sqrt{\log(p^2/\delta)}(1+c_0)<1$ implies $\frac{32sc}{\kappa} \sqrt{\log (p^2 / \delta)} \Big ( z_j +\frac{1 }{t_0^{1/2}}
\Big )< 1$ for $t\geq t_0$. 
By recursively applying the above process and adopting a similar analysis in part (a), we have 
\begin{equation*}
    \Vert \nabla {L}_t(\beta^\ast;\hat \beta_{t-1} )\Vert_\infty  \leq \widetilde \cO \Big ( \frac{1}{\kappa\sqrt{t}} \Big ).
\end{equation*}
By combining the above result with the $\ell_1$-norm error bound $\| \hat \beta_t - \beta^*\|_1 \leq \widetilde \cO(\frac{s}{\kappa \sqrt{t}})$ and \eqref{eq:2normerror}, we conclude that with probability at least $1-4\delta$, 
\begin{equation*}
\begin{split}
\| \hat \beta_t - \beta^*\|_2  \leq \widetilde \cO \Big (  \frac{\sqrt{s}}{\kappa\sqrt{t}} \Big ).
\end{split}
\end{equation*}
This completes the proof. 
\end{proof}

\begin{lemma}\label{lemma:series}
Let $\{z_j\}$ be a sequence satisfying $ z_j \leq \frac{1}{\sqrt{j}}$.  Then for any $q < 1$, there exists a constant $C>0$ such that 
 $$\sum_{j=1}^t q^{t-j} z_j  \leq \frac{C\ln t}{\sqrt{t}}.$$
\end{lemma}
\begin{proof}[Proof of Lemma \ref{lemma:series}] Our proof consists of two steps. In Step 1, we show that
$$\sum_{j=1}^t q^{t-j} z_j \leq \sum\limits_{j=1}^t\frac{q^{t-j}}{\sqrt{j}} \leq \frac{q^{t-1}}{-2\ln q}+q^t\int^{t+1}_1\frac{1}{\sqrt{x}q^x}dx.$$

Let $f(j)=\ln(\frac{1}{q^j\sqrt{j}})$, whose derivative is $f'(j)=-\ln q -\frac{1}{2j}$. A zero of $f'(j)$ is $\hat{j}=-\frac{1}{2\ln q}$, which indicates that $f(j)$ is  monotonically decreasing for $j \leq \hat{j}$ and monotonically increasing for $j \geq \hat{j}$. Therefore, we have
\begin{equation*}
\begin{aligned}
\sum_{j=1}^t q^{t-j} z_j  & \leq q^t\sum\limits^{\lfloor\hat{j}\rfloor}_{j=1}\frac{1}{q^j\sqrt{j}}+q^t\int^{t+1}_{\lceil\hat{j}\rceil}\frac{1}{\sqrt{x}q^x}dx 
\\
& \leq q^t\frac{\lfloor\hat{j}\rfloor}{q}+q^t\int^{t+1}_{\lceil\hat{j}\rceil}\frac{1}{\sqrt{x}q^x}dx \leq \frac{-\frac{1}{2\ln q}}{ q^{1-t}}+ q^t\int^{t+1}_{1}\frac{1}{\sqrt{x} q^x}dx,
\end{aligned}
\end{equation*}
which completes Step 1. 

In Step 2, we show 
$$ q^t\int^{t+1}_{1}\frac{1}{\sqrt{x} q^x}dx \leq \frac{2}{ q\ln(\frac{1}{ q})\sqrt{t+1}}+\frac{ q^{t-1}}{\ln q}+ \frac{ q^{\frac{3t}{4}}}{\ln(\frac{1}{ q}) q^{\frac{1}{4}}}.$$
We first use integration by parts and obtain
$$ q^t\int^{t+1}_{1}\frac{1}{\sqrt{x} q^x}dx \leq \frac{1}{- q\ln q\sqrt{t+1}}+\frac{ q^{t-1}}{\ln q}+ \frac{ q^t}{2\ln(\frac{1}{ q})}\int^{t+1}_{1}\frac{1}{\sqrt{x}^3 q^x}dx,
$$ 
where the last term in the right-hand side of the above inequality can be splitted into two parts
$$\frac{ q^t}{2\ln(\frac{1}{ q})}\int^{t+1}_{1}\frac{1}{\sqrt{x}^3 q^x}dx = \frac{ q^t}{2\ln(\frac{1}{ q})}\int^{t+1}_{\frac{t+1}{4}}\frac{1}{\sqrt{x}^3 q^x}dx+\frac{ q^t}{2\ln(\frac{1}{ q})}\int^{\frac{t+1}{4}}_{1}\frac{1}{\sqrt{x}^3 q^x}dx.$$
For the first part, we have 
$$\frac{ q^t}{2\ln(\frac{1}{ q})}\int^{t+1}_{\frac{t+1}{4}}\frac{1}{\sqrt{x}^3 q^x}dx \leq \frac{1}{2 q\ln(\frac{1}{ q})}\int^{t+1}_{\frac{t+1}{4}}\frac{1}{\sqrt{x}^3}dx=\frac{1}{ q\ln(\frac{1}{ q})}\frac{1}{\sqrt{t+1}}.$$
For the second part, we have
$$\frac{ q^t}{2\ln(\frac{1}{ q})}\int^{\frac{t+1}{4}}_{1}\frac{1}{\sqrt{x}^3 q^x}dx \leq \frac{ q^{\frac{3t-1}{4}}}{2\ln(\frac{1}{ q})}\int^{\frac{t+1}{4}}_{1}\frac{1}{\sqrt{x}^3}dx=\frac{ q^{\frac{3t-1}{4}}}{\ln(\frac{1}{ q})}(1-\frac{2}{\sqrt{t+1}})\leq \frac{ q^{\frac{3t-1}{4}}}{\ln(\frac{1}{ q})}=\frac{ q^{\frac{3t}{4}}}{\ln(\frac{1}{ q}) q^{\frac{1}{4}}}.$$
Adding the above bounds together finishes the proof in Step 2. 

Combining Steps 1 and 2 acquires 
$$\sum_{j=1}^t  q^{t-j} z_j \leq -\frac{ q^{t-1}}{2\ln\frac{1}{ q}} + \frac{2}{ q\ln(\frac{1}{ q})\sqrt{t+1}}+ \frac{ q^{\frac{3t}{4}}}{\ln(\frac{1}{ q}) q^{\frac{1}{4}}} \leq \frac{2}{ q\ln(\frac{1}{ q})\sqrt{t+1}}+ \frac{ q^{\frac{3t}{4}}}{\ln(\frac{1}{ q}) q^{\frac{1}{4}}}.$$
It is clear that $\frac{2}{ q\ln(\frac{1}{ q})\sqrt{t+1}} \leq C_1\frac{\ln t}{\sqrt{t}} $ for some $C_1>0$. Note that $ q^{\frac{3t}{4}}>0$, thus $g(t)=\frac{\ln(\frac{1}{ q}) q^{\frac{1}{4}}\ln(t)}{\sqrt{t} q^{\frac{3t}{4}}}>0$ holds for all $t\geq 2$. Moreover, because $\lim_{t\rightarrow\infty}g(t)=\infty $, $g(t)$ can reach its minimum point when $t=t^*\in\{2,3,\cdots\}$. By setting $C_2=\frac{1}{g(t^*)}$,  we have $C_2g(t)=C_2\frac{\ln(t)}{\sqrt{t}}/\frac{ q^{\frac{3t}{4}}}{\ln(\frac{1}{ q}) q^{\frac{1}{4}}}\geq1$, further leading to  $\frac{ q^{\frac{3t}{4}}}{\ln(\frac{1}{ q}) q^{\frac{1}{4}}}\leq C_2\frac{\ln(t)}{\sqrt{t}}$. Letting $C=C_1+C_2$, we conclude 
$$\sum_{j=1}^t  q^{t-j} z_j  \leq \frac{C\ln t}{\sqrt{t}},$$ which completes the proof. 
\end{proof}

\section{Additional Numerical Results}
\label{sec:more_numerics}

In this section, we conduct additional numerical experiments for instances where (i) the covariate $x_i$ is generated under different correlation designs, (ii) the underlying $\beta^*$ preserves difference sparsity levels, and (iii) the underlying $\beta^*$ is generated with strong signal designs. Specifically, we consider the following three problem setups:
\begin{enumerate}
    \item[(a)] The covariate $x_i$ is generated under different correlation setups, namely independent correlation, Toeplitz  $\rho =0.3$ correlation, and $\rho=0.7$ correlation. The underlying $\beta^*$ is generated under the weak signal setup as  in Section~\ref{sec:experiments}. 
    \item[(b)] The underlying $\beta^*$ has different sparsity levels $s = \|\beta^*\|_0 \in\{10,50 \}$ and each nonzero entry of $\beta^*$ is independently generated such that $\beta_j^* \sim \cN(0,0.25)$ for all $j \in \{1,2,\cdots, s\}$. Each entry of the covariate $x_i$ is independently generated such that  $x_{ij} \sim \cN(0,1)$ for all $j \in [p]$. 
    \item[(c)] The underlying $\beta^*$ has strong signals with sparsity level $s = 20$, where each nonzero entry of $\beta^*$ is independently generated such that  $\beta_j \sim \cN(0,4)$ for all $j \in S$.
\end{enumerate}
Moreover, other primitives and algorithm parameters such as step-sizes $\eta_t$  are the same as those in the experiments of Section~\ref{sec:experiments}.

\begin{figure}[t]
  \centering
  \subfigure[Independent]{
  \includegraphics[width=0.3\textwidth]{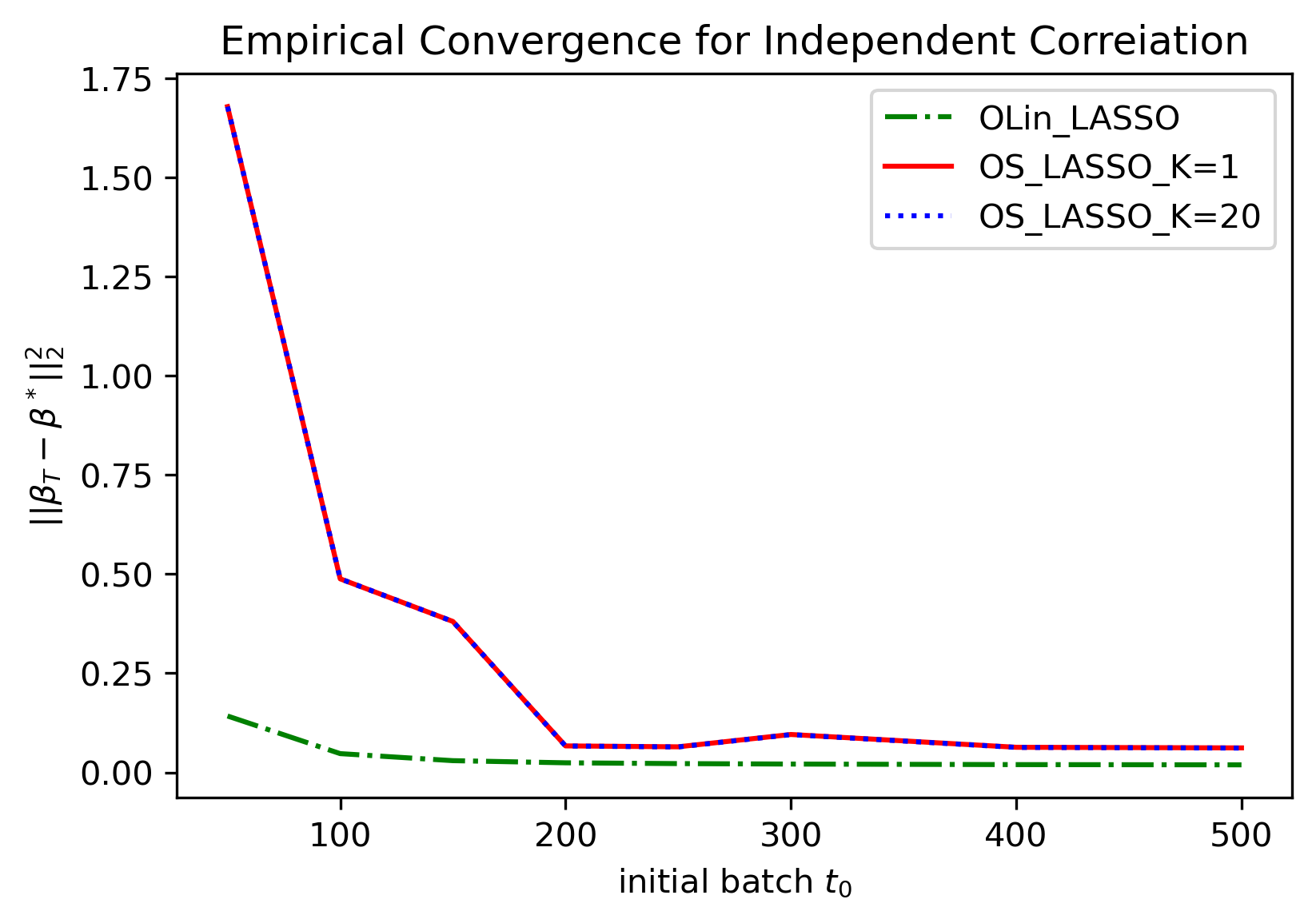}}
  \subfigure[Toeplitz $\rho =0.3$]{
  \includegraphics[width=0.3\textwidth]{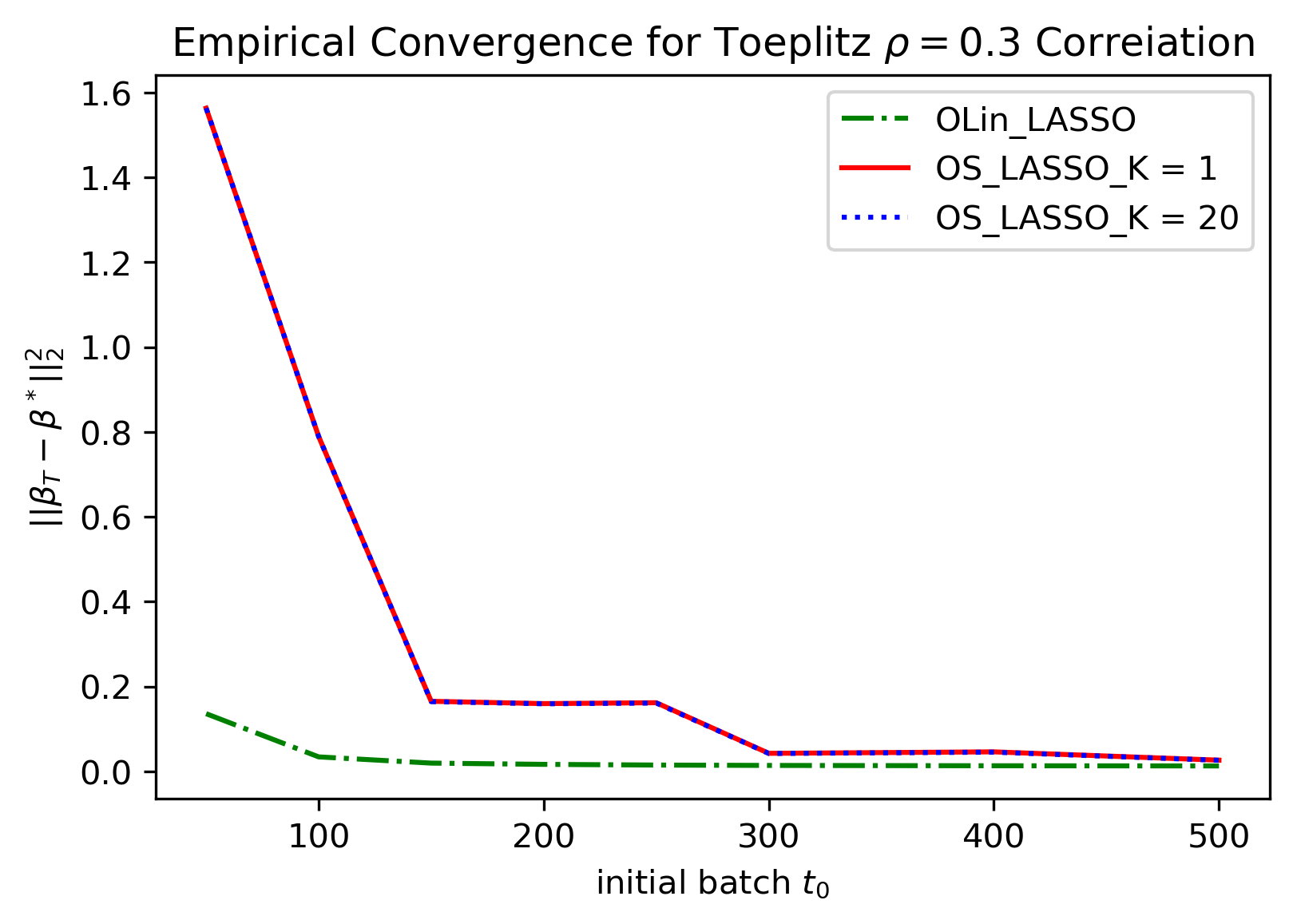}}
  \subfigure[Toeplitz $\rho =0.7$]{
  \includegraphics[width=0.3\textwidth]{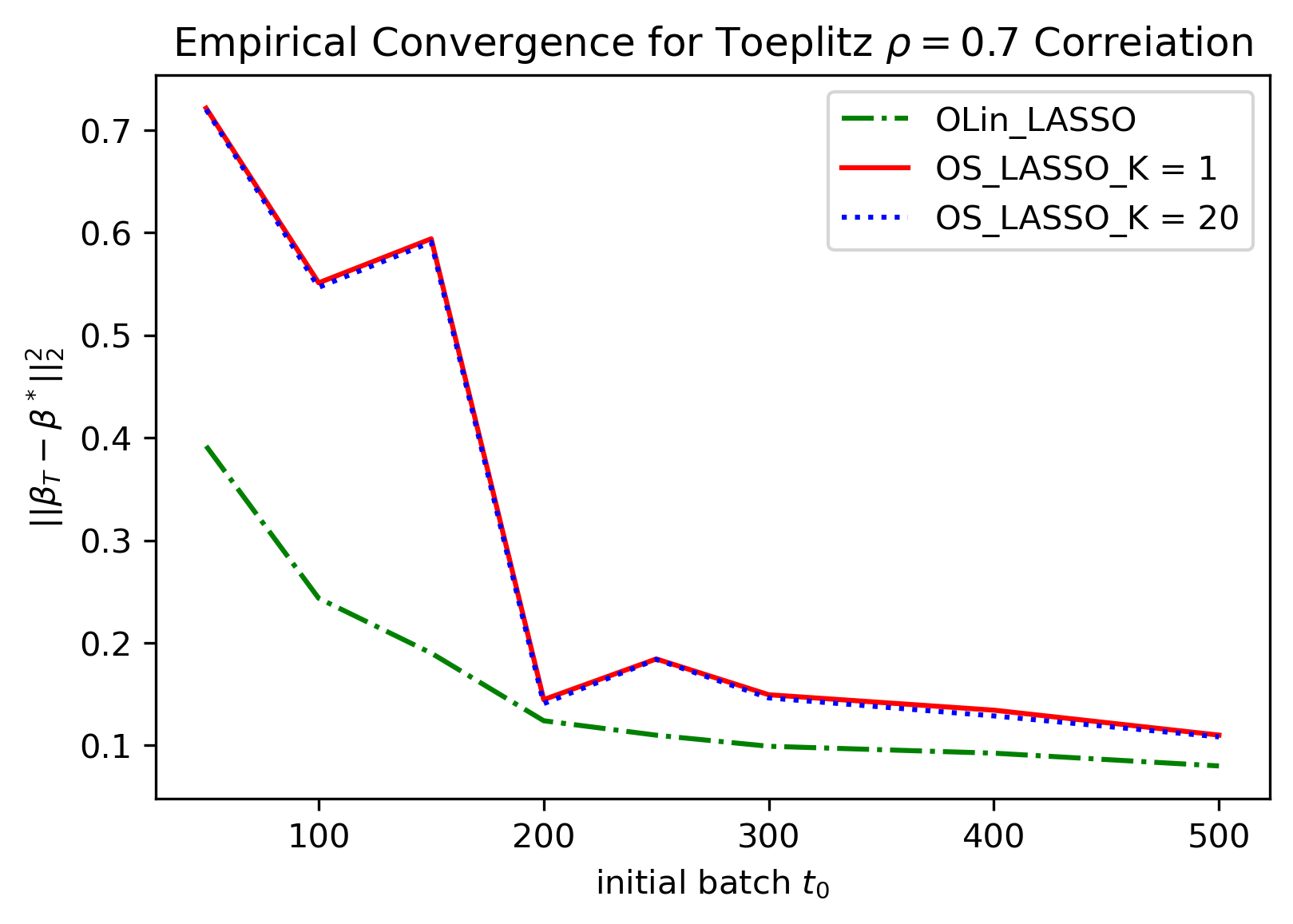}}
  \caption{Empirical Convergence of MSE$\| \beta_T - \beta^* \| ^2_2$ under weak signal setup for $t_0 =50,100,150,\cdots,500$ and $T=10000$ online learning rounds for Independent, Toeplitz $\rho =0.3$ and $\rho=0.7$  correlation designs}
  \label{fig:Toeinit}
\end{figure}

\begin{figure}[t]
  \centering
  \subfigure[Independent]{
  \includegraphics[width=0.3\textwidth]{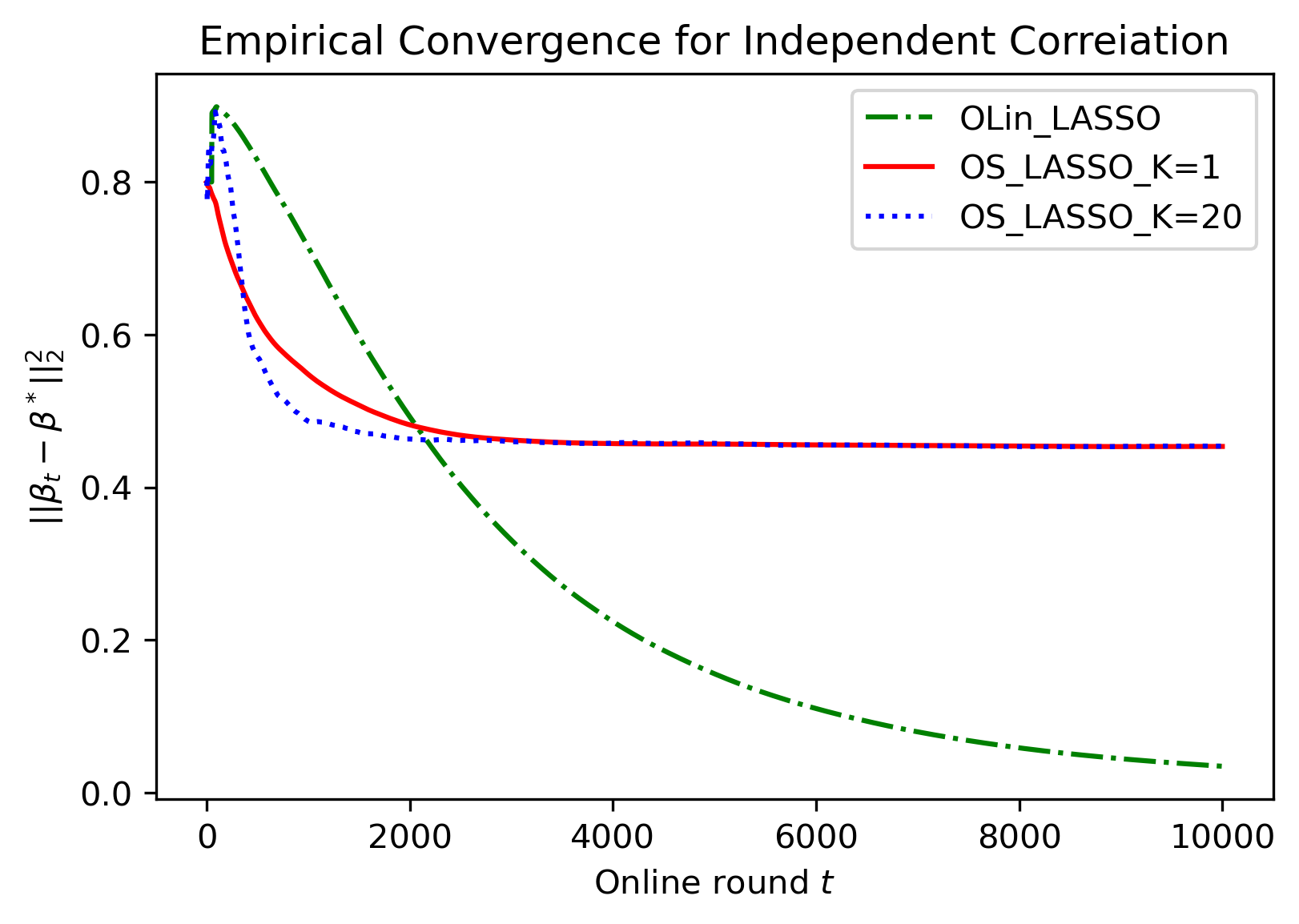}}
  \subfigure[Toeplitz $\rho =0.3$]{
  \includegraphics[width=0.3\textwidth]{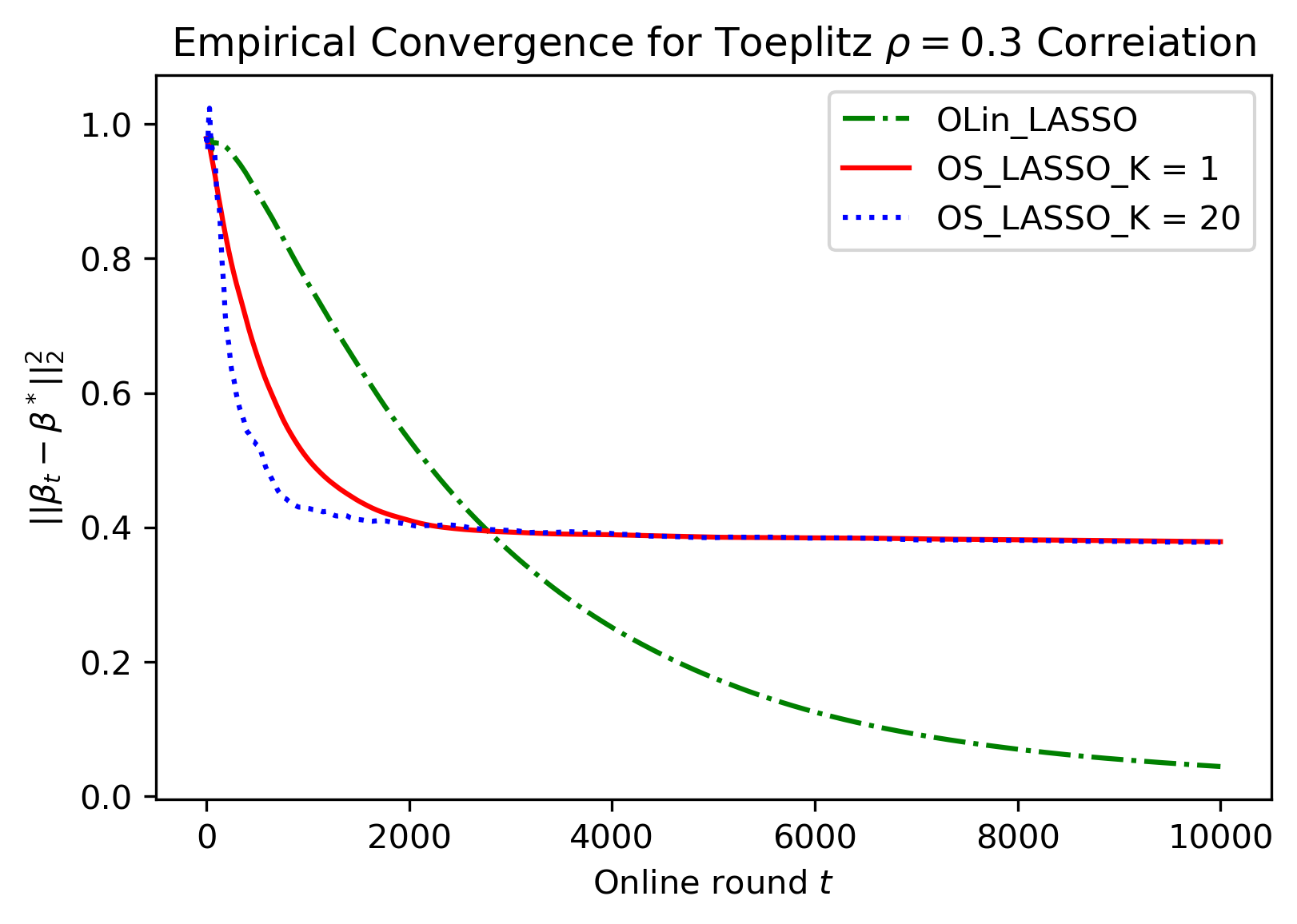}}
  \subfigure[Toeplitz $\rho =0.7$]{
  \includegraphics[width=0.3\textwidth]{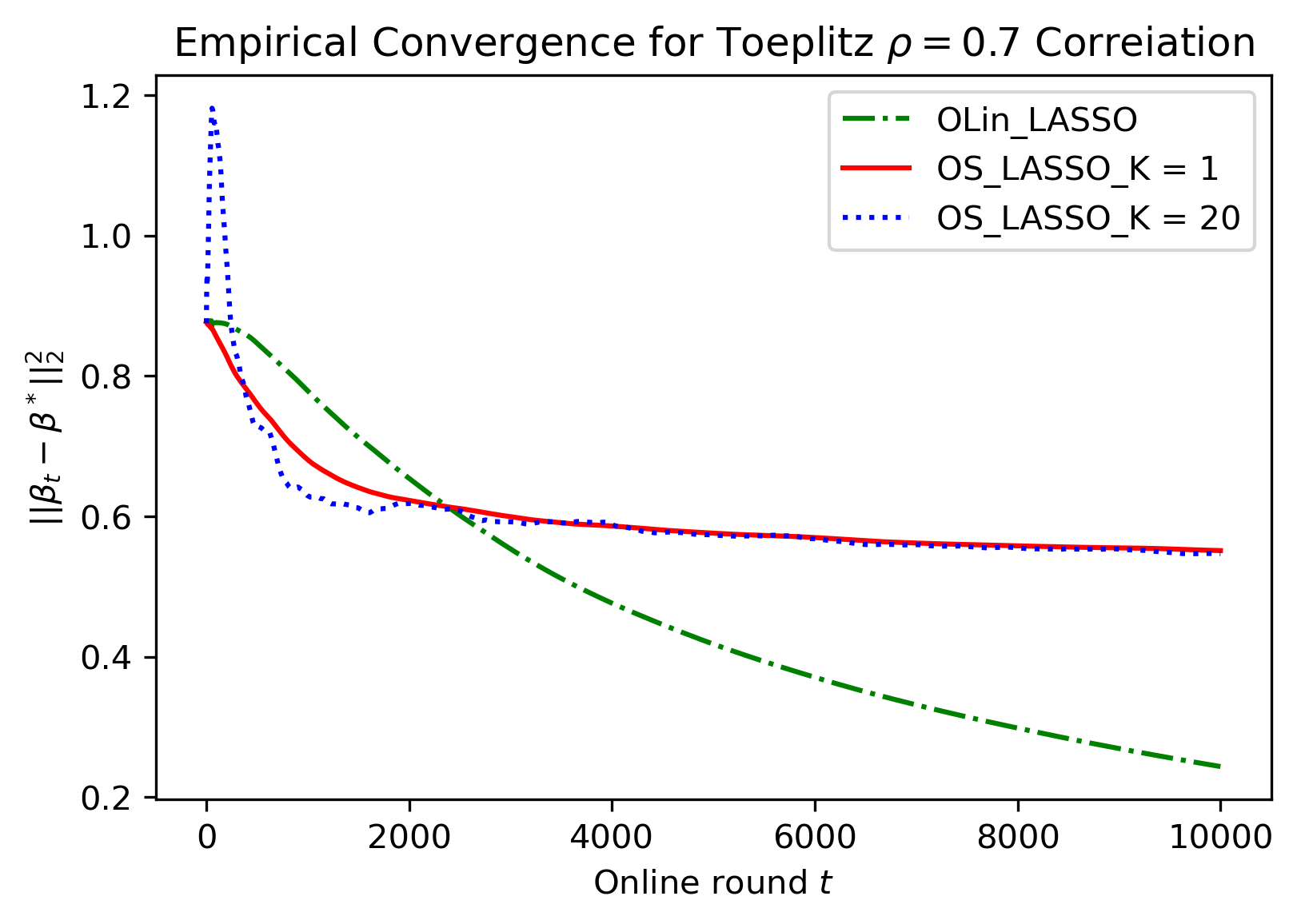}}
  \caption{Empirical Convergence of MSE$\| \beta_t - \beta^* \| ^2_2$ under weak signal setup for online learning rounds $t \in [0,10000]$ and initial batch size $t_0 = 100$  for Independent, Toeplitz $\rho =0.3$ and $\rho=0.7$ correlation designs}
  \label{fig:Toeonline}
\end{figure}
For each problem setup mentioned above, we first run our algorithm and the baselines  for $T= 10^4$ online rounds with different initial batch size $t_0 = 50,100,\cdots, 500$, and plot the empirical MSE $\| \beta_T - \beta^*\|_2^2$ against $t_0$.  We then fix $t_0 = 100$ and plot $\| \beta_t - \beta^*\|_2^2$ against the online round $t \in [0,10000]$. We report the results under setup (a) in Figures~~\ref{fig:Toeinit} and~\ref{fig:Toeonline}, report the results under setup (b) in Figures~\ref{fig:spainit} and~\ref{fig:spaonline}, and provide the results under setup (c) in Figure~\ref{fig:Strong}.

\begin{figure}[t]
  \centering
  \subfigure[sparsity: $s = 10$]{
  \includegraphics[width=0.4\textwidth]{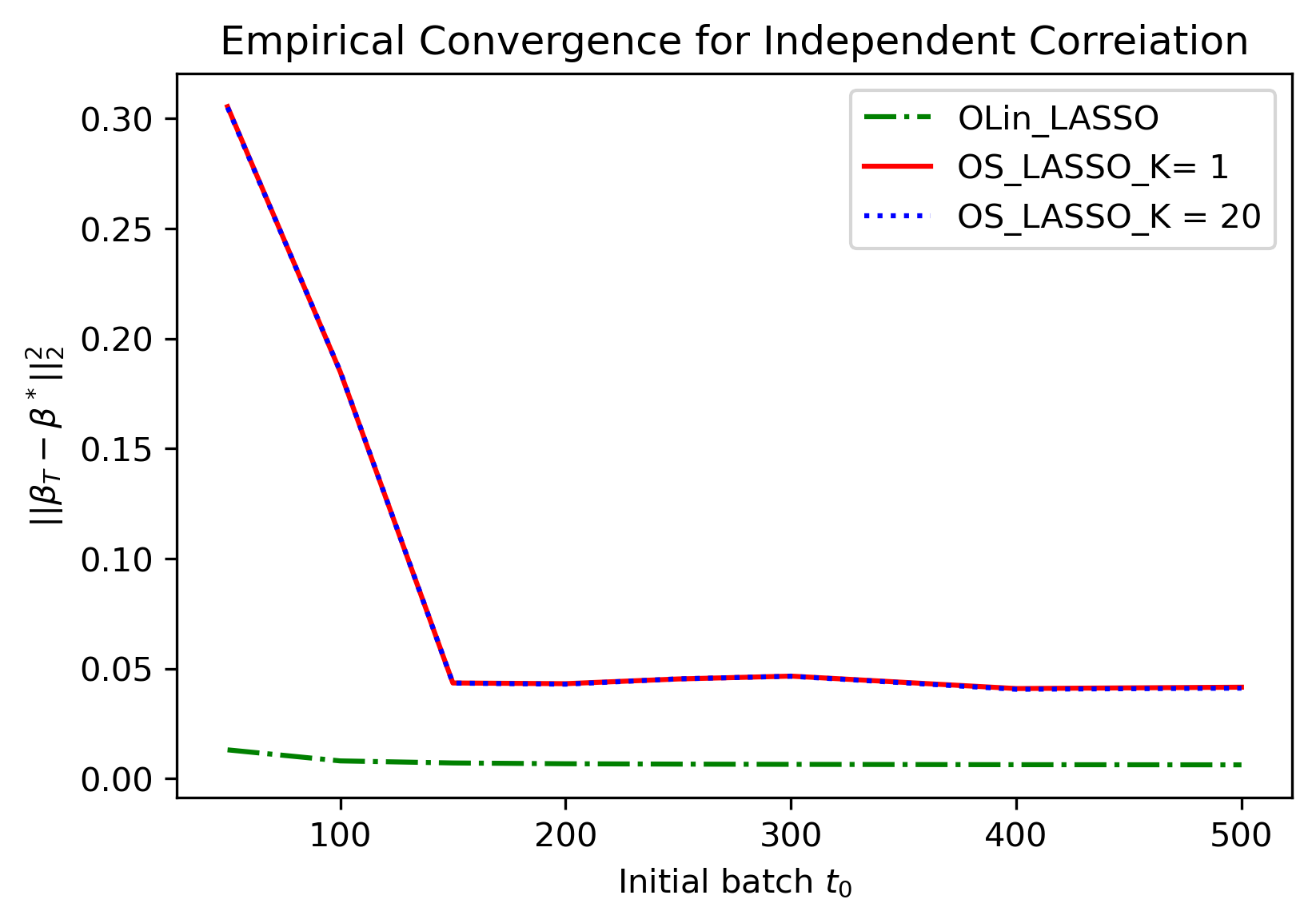}}
  \subfigure[sparsity: $s = 50$]{
  \includegraphics[width=0.4\textwidth]{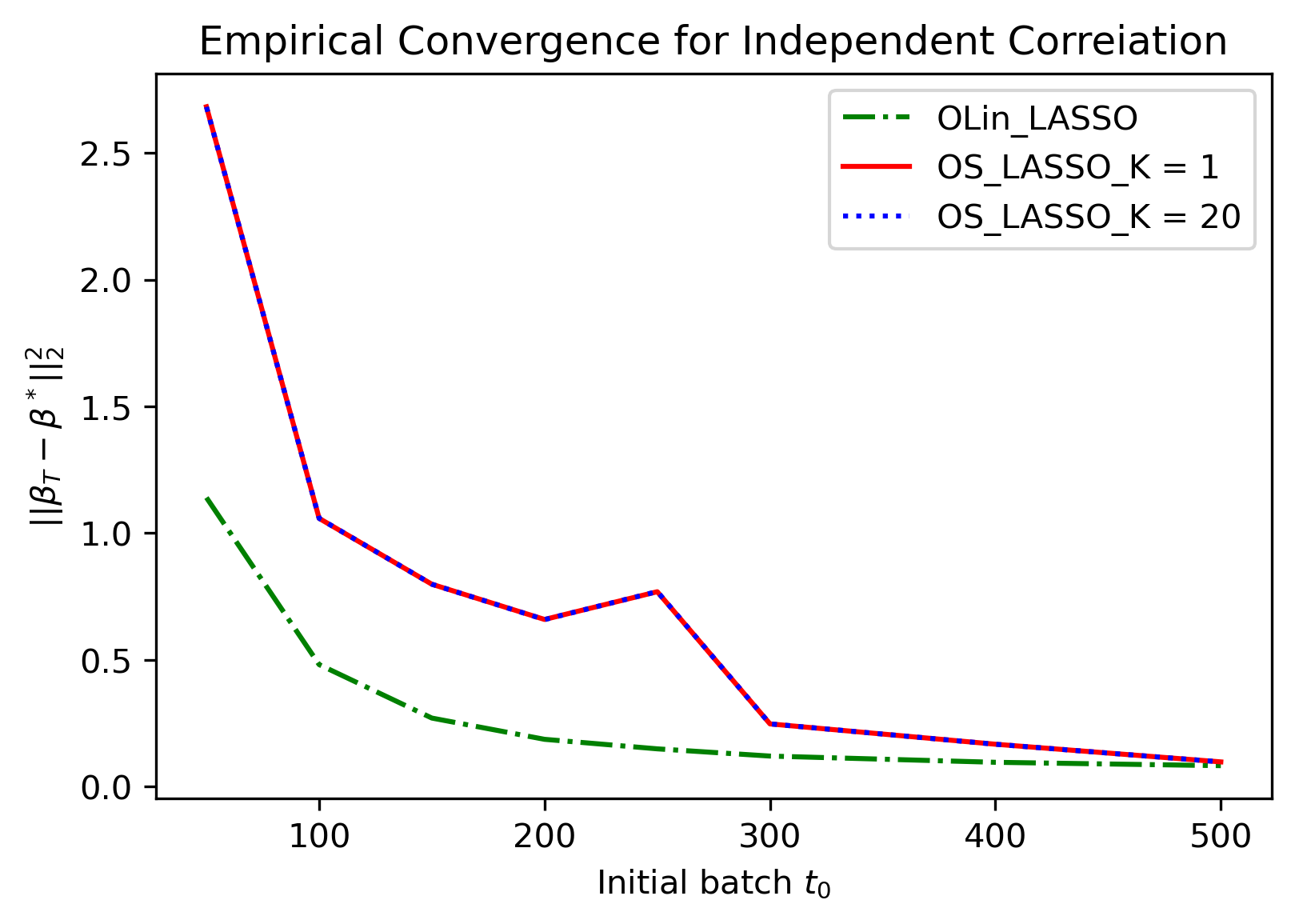}}
  \caption{Empirical Convergence of MSE $\| \beta_T - \beta^* \| ^2_2$ under weak signal setup for $t_0 =50,100,150,\cdots,500$ and $T=10000$ online learning rounds for sparsity level $s = 10$ and $s = 50$}
  \label{fig:spainit}
\end{figure}

\begin{figure}[t]
  \centering
  \subfigure[sparsity: $s = 10$]{
  \includegraphics[width=0.4\textwidth]{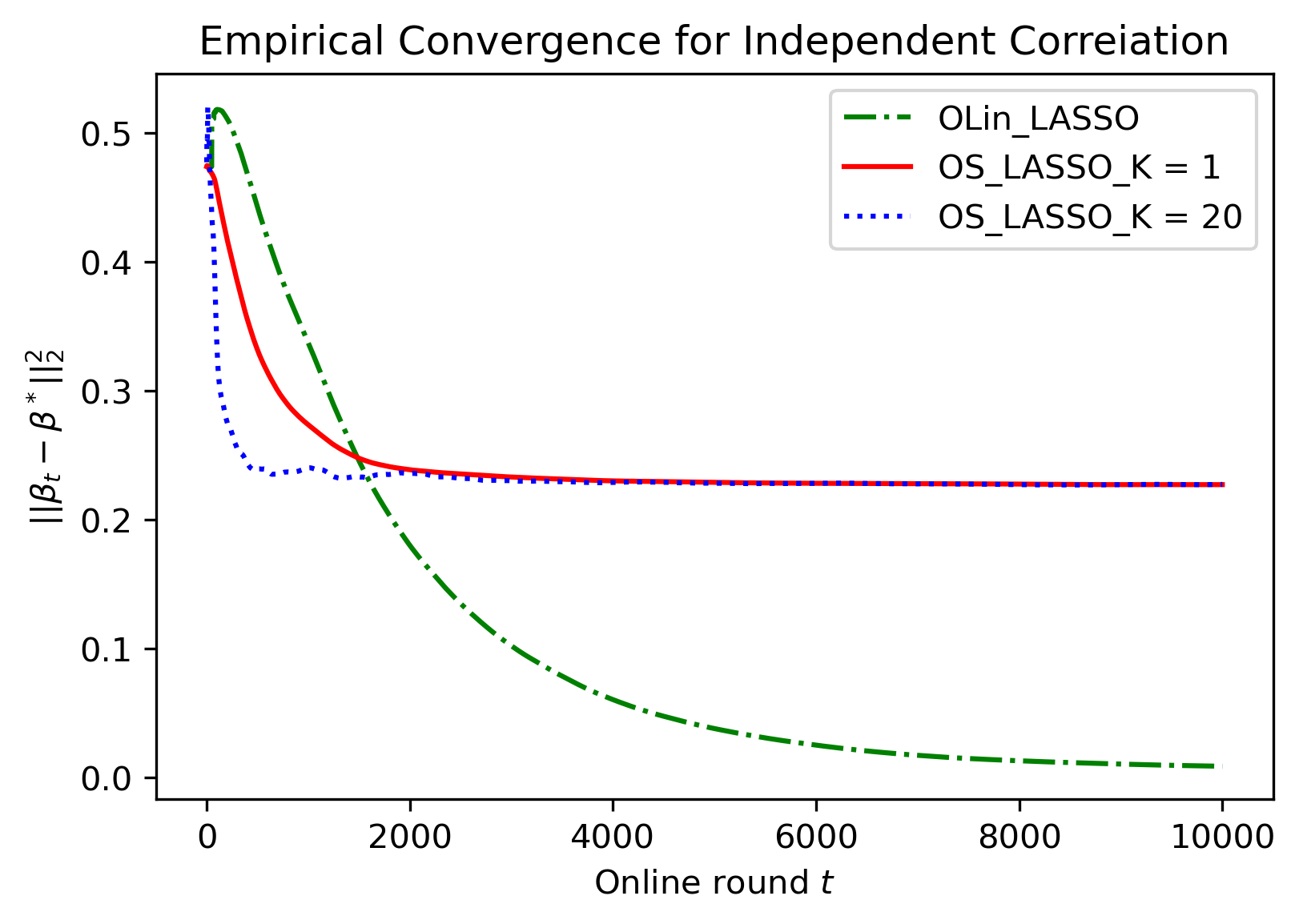}}
  \subfigure[sparsity: $s = 50$]{
  \includegraphics[width=0.4\textwidth]{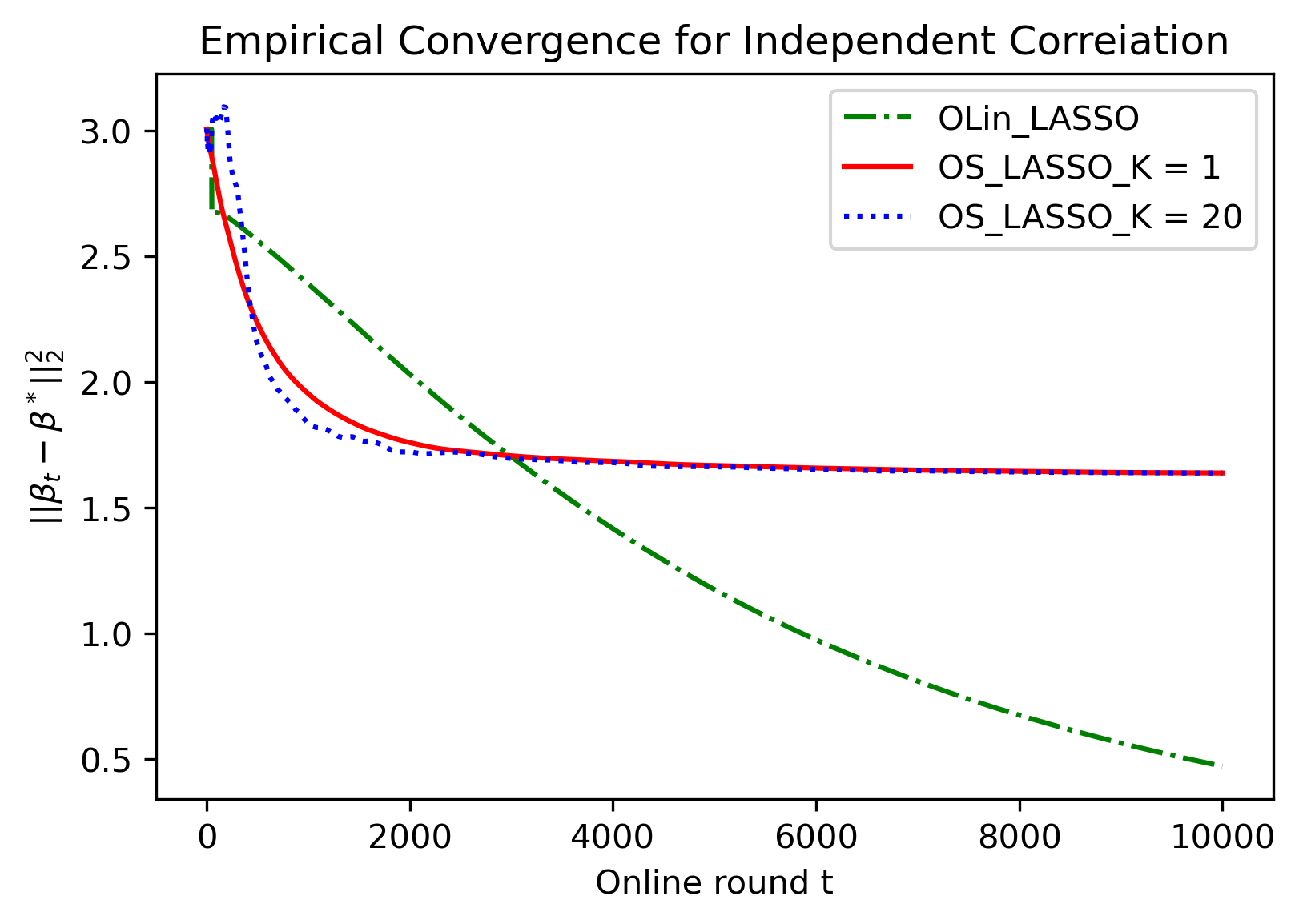}}
  \caption{Empirical Convergence of MSE $\| \beta_t - \beta^* \| ^2_2$ under weak signal setup for online learning rounds $t \in [0,10000]$ and initial batch size $t_0 = 100$  for sparsity level $s = 10$ and $s = 50$}
  \label{fig:spaonline}
\end{figure}

\begin{figure}[H]
  \centering
   \subfigure[online learning rounds $T = 10^4$]{
  \includegraphics[width=0.4\textwidth]{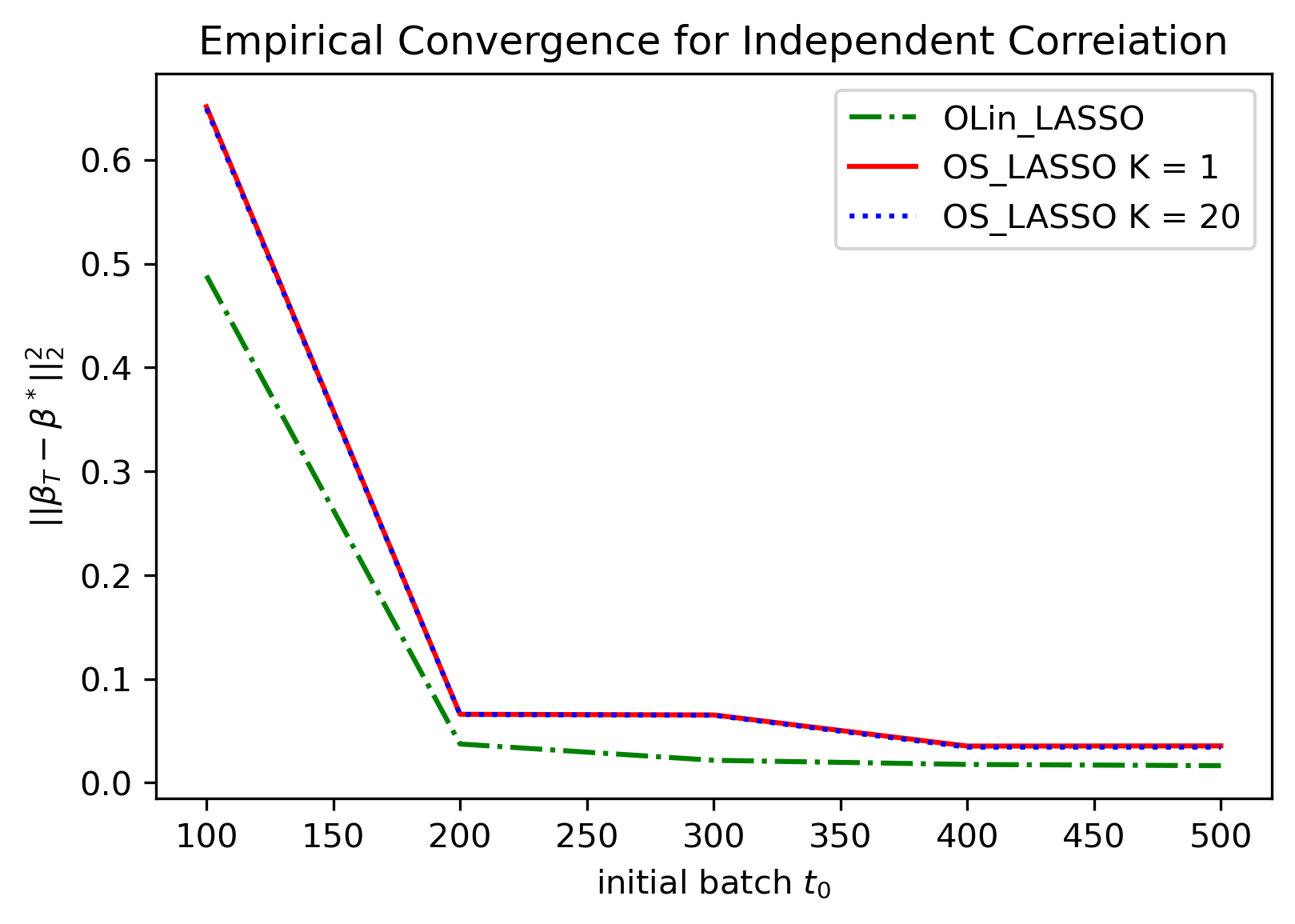}}
    \subfigure[initial batch $t_0 = 200$]{
  \includegraphics[width=0.4\textwidth]{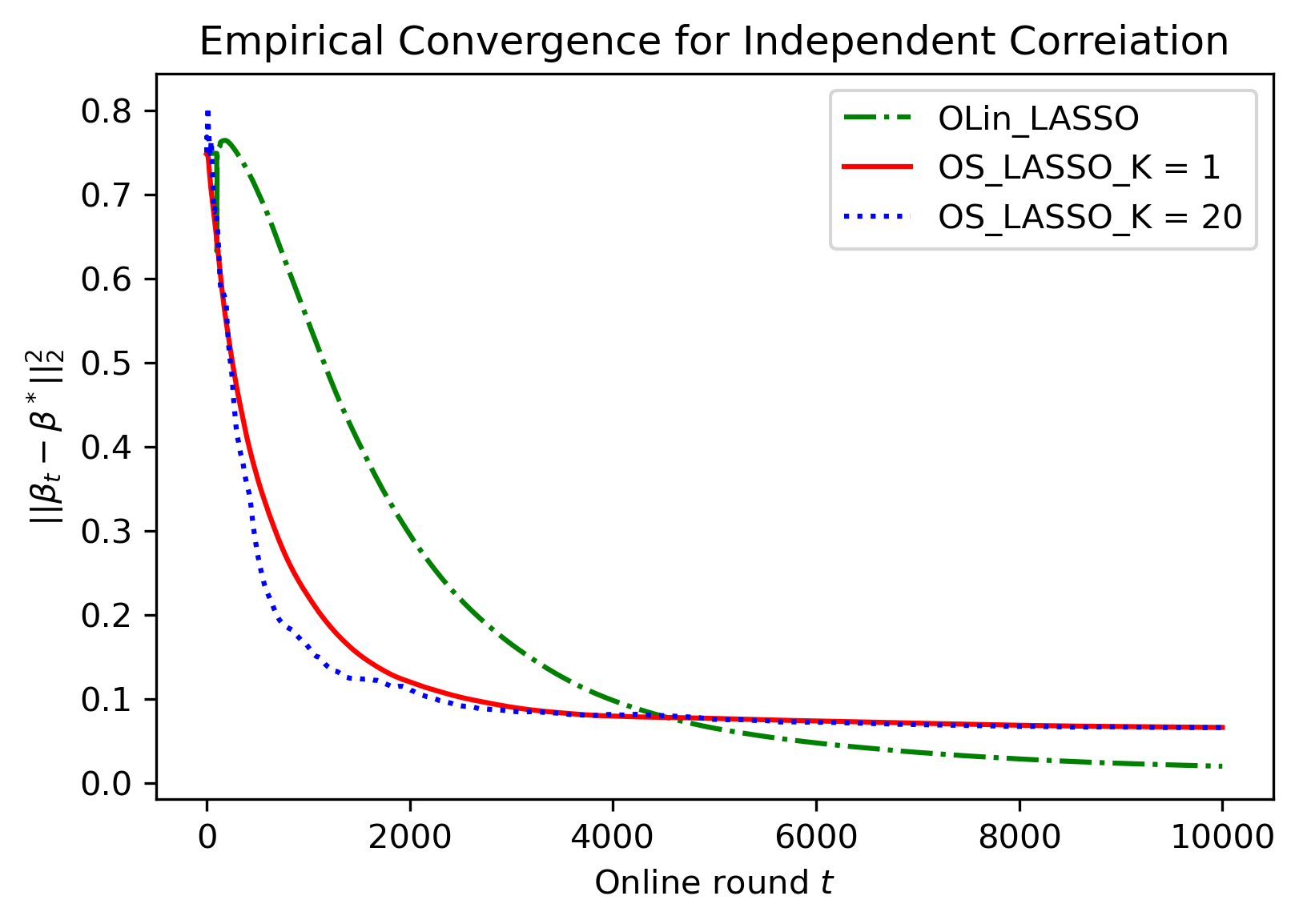}}
  \caption{Empirical convergence of MSE under the strong signal setup}
  \label{fig:Strong}
\end{figure}

Figures ~\ref{fig:Toeinit} and~\ref{fig:Toeonline} show that our algorithm outperforms the baseline algorithms  under other types of covariate correlation designs. We also observe from 
Figures~\ref{fig:spainit} and~\ref{fig:spaonline}  that our algorithm exhibits superior performance to the baseline algorithms under different sparsity levels. These observations demonstrate the practical efficiency of our algorithm under the weak signal setup. Further, for the strong signal setup (c), Figure~\ref{fig:Strong}(b) suggests that our algorithm outperforms the baselines when the initial batch size is small ($t_0 \leq 200$). However, when the initial batch size gets larger, our algorithm exhibits a comparable performance with the baselines, where the corresponding  MSEs $\| \beta_T-\beta^* \| ^2_2$ are around $0.03$ and $0.06$, respectively.  This is because under strong signals, an initial batch of size $t_0 \geq 200$ is sufficient for the baseline algorithms {\tt OS\_LASSO} with different $K$s to identify the true support of $\beta^*$ through solving an offline LASSO using the initial batch. In contrast, as shown in Figure~\ref{fig:online}, the initial batch size $t_0 = 200$ is insufficient for the baselines to identify the true support of $\beta^*$ under the weak signal setup.

\end{document}